\documentclass{article}

\usepackage[preprint]{neurips_2026}

\usepackage[utf8]{inputenc} % allow utf-8 input
\usepackage[T1]{fontenc}    % use 8-bit T1 fonts
\usepackage{algorithm}
\usepackage{algpseudocode}
\usepackage{hyperref}       % hyperlinks
\usepackage{url}            % simple URL typesetting
\usepackage{booktabs}       % professional-quality tables
\usepackage{amsfonts}       % blackboard math symbols
\usepackage{amsmath}
\allowdisplaybreaks
\usepackage{amssymb}
\usepackage{amsthm}
\usepackage{dsfont}
\usepackage{cancel}
\usepackage{mathtools}
\usepackage{nicefrac}       % compact symbols for 1/2, etc.
\usepackage{microtype}      % microtypography
\usepackage[dvipsnames]{xcolor}         % colors
\theoremstyle{plain}
\newtheorem{theorem}{Theorem}[section]
\newtheorem{proposition}[theorem]{Proposition}

\theoremstyle{definition}

\newcommand{\A}{\mathcal{A}}

\newcommand{\G}{\mathcal{G}}

\newcommand{\V}{\mathcal{V}}
\newcommand{\E}{\mathcal{E}}
\newcommand{\X}{\mathcal{X}}
\newcommand{\Y}{\mathcal{Y}}

\newcommand{\J}{\mathcal{J}}
\newcommand{\R}{\mathbb{R}}

\newcommand{\fB}{\mathfrak{B}}
\newcommand{\di}{\mathrm{d}}
\newcommand{\0}{\mathbf{0}}

\newcommand{\Q}{\mathbb{Q}}

\newcommand{\pa}[1]{\mathrm{pa}(#1)}
\newcommand{\ch}[1]{\mathrm{ch}(#1)}
\newcommand{\kld}{D_{\mathrm{KL}}}

\renewcommand{\L}{\mathcal{L}}
\renewcommand{\P}{\mathbb{P}}
\renewcommand{\S}{\mathcal{S}}

\title{Neural Backward Filtering Forward Guiding}

\author{%
  Gefan Yang \\
  Department of Computer Science \\
  University of Copenhagen \\
  Universitesparken 1, 2100 København, Denmark \\
  \texttt{gy@di.ku.dk} \\
  \And
  Frank van der Meulen \\
  Department of Mathematics \\
  Vrije Universiteit Amsterdam \\
  De Boelelaan 1111, 1081HV Amsterdam, The Netherlands \\
  \texttt{f.h.van.der.meulen@vu.nl} \\
  \AND
  Stefan Sommer \\
  Department of Computer Science \\
  University of Copenhagen \\
  Universitesparken 1, 2100 København, Denmark \\
  \texttt{sommer@di.ku.dk}
}

\begin{document}

\maketitle

\begin{abstract}
    Inference in nonlinear continuous stochastic processes on trees is challenging, particularly when observations are sparse and the topology is complex. Exact smoothing via Doob's $h$-transform is intractable for general nonlinear dynamics. We propose Neural Backward Filtering Forward Guiding (NBFFG), a unified framework for both discrete transitions and continuous diffusions. Our method constructs a variational posterior by leveraging a proxy linear-Gaussian process. This proxy process yields a closed-form backward filter that serves as a guide, steering the generative path toward high-likelihood regions. We then learn a neural residual to capture the non-linear discrepancies. This formulation allows for an unbiased pathwise subsampling scheme, reducing the training complexity from tree-size dependent to path-length dependent. Empirical results show that NBFFG outperforms baselines on synthetic benchmarks, and we demonstrate the method on a high-dimensional inference task in phylogenetic analysis with reconstruction of ancestral butterfly wing shapes.
\end{abstract}

\section{Introduction}
Tree-structured stochastic processes are fundamental for modeling hierarchical dependencies in complex systems, from phylogenetic evolution in biology \citep{felsenstein_evolutionarytreesdna_1981, huelsenbeck_stochastic_2003} and spin-glass state spaces in statistical physics \citep{bovier_gaussian_2017} to cascading diffusion in social networks \citep{kempe_maximizing_2003}. In these settings, a central inferential challenge is smoothing: reconstructing latent internal states and trajectories from partially available observations. Exact Bayesian smoothing is characterized by Doob's $h$-transform, where future observations enter through a harmonic $h$-function computed by collecting evidence from leaves to the root, equivalently by belief propagation or sum-product message passing. While this backward pass is analytically tractable for simple Linear-Gaussian (LG) models, it becomes computationally prohibitive for general nonlinear dynamics.

Current approaches often rely on Sequential Monte Carlo (SMC) or Markov chain Monte Carlo (MCMC). However, particle methods can suffer severe weight degeneracy as the hierarchy deepens, and MCMC can mix slowly when distant observations provide weak global information to internal states. We introduce Neural Backward Filtering Forward Guiding (NBFFG), a variational framework that approximates the intractable smoothing posterior by correcting a tractable guided proposal. The proposal is built from a cheap proxy backward filter, which replaces the exact Doob's $h$-function with a closed-form approximation and steers samples toward regions compatible with descendant observations. This proxy need not be accurate: it is only required to provide coarse global information. A neural residual then corrects the induced proposal over the tree, with parameterizations adapted to discrete transitions and continuous diffusion edges.

Our contributions are:
\begin{itemize}
    \item We formulate tree smoothing as neural correction of a guided proposal: a cheap proxy backward filter provides global descendant information, while a neural residual corrects the induced posterior, even under rough or misspecified proxies.
    \item We give a unified variational construction for discrete tree transitions and continuous diffusion edges, instantiated as Gaussian-mixture corrections and controlled SDEs.
    \item We introduce path-subsampled training with amortized correction over root-to-node contexts, reducing each update from the full tree to sampled root-to-leaf paths.
\end{itemize}
We validate these claims on exact-posterior LG benchmarks, nonlinear multimodal systems, and high-dimensional phylogenetic shape inference.

\section{Problem Setup}

% \vspace{-5pt}
\subsection{Generative process on the tree}
% \vspace{-5pt}
Let $\G=(\V,\E)$ be a directed tree rooted at a fixed vertex $\0$. We write $\V^+=\V\setminus\{\0\}$ for the non-root vertices, $\L\subset\V$ for the leaves, and $\S=\V\setminus(\{\0\}\cup\L)$ for the non-root internal vertices. The root is fixed at $X_\0=x_\0$, and each non-root vertex $v\in\V^+$ carries a random variable $X_v$ taking values in a measurable space $\X_v$. When the root is itself unobserved, we represent it as a vertex $\rho\in\S$ attached to the fixed super-root $\0$; root inference then means inference on $X_\rho$. We denote by $\pa{v}$ and $\ch{v}$ the parent and children of $v$, respectively. Let $\L_v\subseteq \L$ denote the leaves of the subtree rooted at $v$.

All leaves are observed terminal variables. Observations of an internal hidden state are represented by attaching observed leaf children to that state, with multiple observations corresponding to multiple leaves. Thus latent dynamics and observation models are both treated as edge transitions, with leaf edges clamped to their observed child values.

The prior is a Markov generative process on the tree: conditional on a parent state, the subtrees below its children evolve independently. Hence the joint law over non-root variables factorizes edge-wise as
\begin{equation}
    \P(\di x_{\V^+}\mid x_\0)
    =
    \prod_{v\in\V^+} P_v(x_{\pa{v}}, \di x_v),
    \label{eq:prior}
\end{equation}
where $P_v:\X_{\pa{v}}\rightharpoonup\X_v$ is the Markov kernel along $(\pa{v},v)$. This notation covers both transition types---discrete and continuous transitions---considered in the paper. For discrete edges, we consider per-edge Gaussian kernels on $\R^d$,
\begin{equation}
    P_v(x, \di y)
    =
    \mathcal{N}\!\left(\di y; \mu_v(x), \Sigma_v(x)\right).
    \label{eq:discrete_edge}
\end{equation}
For continuous edges, we first distinguish the vertex state $X_v$ on $\R^d$ from the within-edge path $Z_v(t)$ on the continuous path space $\mathcal{C}([0,T_v], \R^d)$. Conditional on $X_{\pa{v}}=Z_v(0)=x$, the path starts at the parent value and ends at the child state:
\begin{equation}
    \di Z_v(t)
    =
    b_v(Z_v(t))\di t
    +
    \sigma_v(Z_v(t))\di W(t),
    \qquad
    Z_v(0)=x,
    \quad
    X_v=Z_v(T_v).
    \label{eq:continuous_edge}
\end{equation}
This induces a per-edge path law $\P^x_v$, and the corresponding vertex-to-vertex kernel is the endpoint law of this diffusion. A measure-theoretic formulation of these kernels is given in Appendix~\ref{app:problem_setup_details}, and an illustration of the generative process on the tree is shown in Figure~\ref{fig:nbffg} (a).

% \vspace{-5pt}
\subsection{Smoothing on the tree}
% \vspace{-5pt}
We now turn from the generative description to the smoothing problem. For each leaf $l\in\L$, assume the leaf-edge kernel is dominated by a reference measure $\nu_l$, so $P_l(x,\di y)=p_l(y\mid x)\nu_l(\di y)$.
In the smoothing problem, the child value on this edge is observed as $X_l=x_l$, and the density above is evaluated at the observed value. Given observed leaf values $x_\L$ and fixed root value $X_\0=x_\0$, the smoothing posterior is the probability measure on unobserved internal realizations $x_{\S}$:
\begin{equation}
    \P^{\star}(\di x_{\S}\mid x_\L, x_\0)
    =
    \frac{1}{Z(x_\L,x_\0)}
    \prod_{l\in\L} p_l(x_l\mid x_{\pa{l}})
    \prod_{s\in\S} P_s(x_{\pa{s}},\di x_s),
    \label{eq:smoothing_posterior}
\end{equation}
where $Z(x_\L,x_\0)$ is the marginal likelihood. The inferential task is to characterize this posterior, especially its marginals at unobserved internal vertices. Except in special cases such as linear models, neither the normalizing constant nor posterior marginals in Equation~\eqref{eq:smoothing_posterior} are available in closed form.
\begin{figure}[ht]
    \centering
    \includegraphics[width=1.0\textwidth]{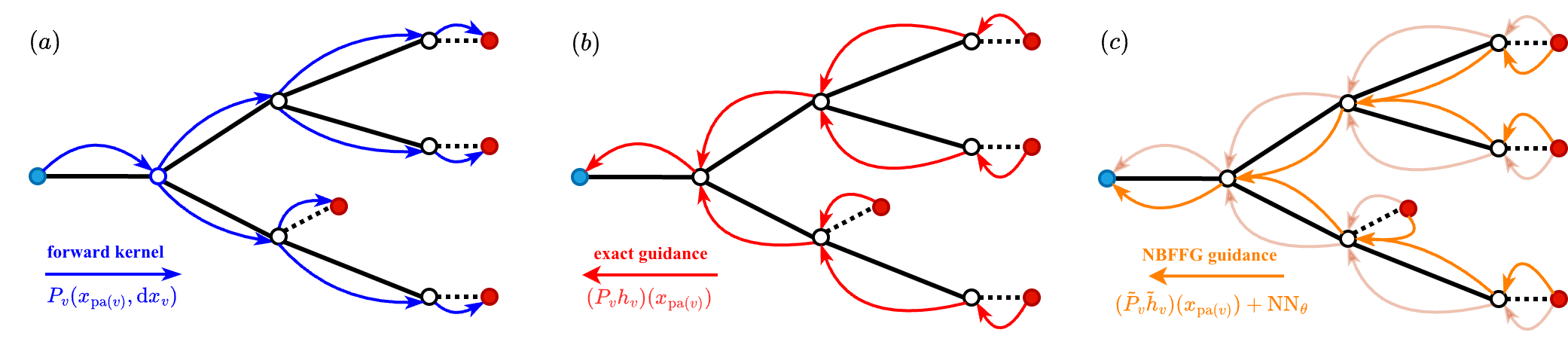}
    \vspace{-10pt}
    \caption{Smoothing on a directed tree with a fixed root $\0$ (blue solid dot). Leaves $\L$ (red solid dots) are observed, internal vertices $\S$ (black hollow dots) are latent, dashed edges denote observation edges, and arrows indicate computation direction: (a) generative sampling from root to leaves; (b) exact backward information filtering from leaves to root; (c) NBFFG framework, where a tractable proxy guide is corrected by a neural residual.}
    \label{fig:nbffg}
\end{figure}

\section{Backward Filtering Forward Guiding}
Under the setup above, exact Bayesian smoothing can be written as message passing on the tree. A backward pass propagates observed leaf information toward the root, producing at each non-leaf vertex the conditional density of its descendant observations. This is the tree analogue of the backward information filter and phylogenetic pruning recursion \citep{felsenstein_evolutionarytreesdna_1981, briers_smoothingalgorithmsstate_2010}, and is equivalently a sum-product or belief-propagation pass on a tree \citep{pearl1988probabilistic,kschischang_factorgraphs_2001}. A subsequent forward pass samples from the root toward the leaves using transitions conditioned on these backward messages.

Let $\S_0:=\S\cup\{\0\}$ denote the vertices that are not observed leaves. For an edge entering an unobserved internal vertex $s\in\S$, the conditioned transition $P^{\star}$ is obtained by a Doob's $h$-transform of the prior transition $P$:
\begin{subequations}
    \begin{align}
        P^{\star}_s(x, \di y)
        &=
        \frac{h_s(y)}{m_s(x)}
        P_s(x, \di y),
        \qquad
        m_s(x):=(P_s h_s)(x), \label{eq:edge_posterior}\\
        h_u(x)
        &=
        \prod_{v\in\ch{u}} m_v(x),
        u\in\S_0, \qquad
        m_v(x)
        :=
        \begin{cases}
            (P_v h_v)(x), & v\in\S,\\
            p_v(x_v\mid x), & v\in\L,
        \end{cases} \label{eq:bif_message}
    \end{align}
\end{subequations}
where the pullback notation $(P_s h_s)$ is defined by Equation~\eqref{eq:kernel_pullback}. Interpret $h_u(x)$ as the density contribution of the observed leaves below $u$ conditioned on $X_u=x$, and $m_v$ as the message sent from child $v$ to its parent. This representation is exact, but evaluating the message $m_s=(P_s h_s)$ is generally intractable. For continuous-time transitions, the same pullback is characterized by the Kolmogorov backward equation:
\begin{equation}
    \left(\frac{\partial}{\partial t} + \A_s \right) \psi_s(t, z) = 0, \qquad \psi_s(T_s, z) = h_s(z), 
    \label{eq:backward_kolmogorov}    
\end{equation}
so that $(P_s h_s)(x) = \psi_s(0, x)$ by the Feynman--Kac formula. Here $\A_s$ is the infinitesimal generator of Equation~\eqref{eq:continuous_edge}, defined by $(\A_s \psi_s)(t, z) := \sum_i [b_s(z)]_i \partial_{i} \psi_s(t, z) + \frac{1}{2}\sum_{i,j} [(\sigma_s\sigma^{\top}_s)(z)]_{ij}\partial^2_{ij} \psi_s(t, z)$.

To address this, \citet{meulenfrankvander_backwardfilteringforward_2025} construct tractable guiding messages $\tilde{m}_s$ by replacing the true transition kernels in the backward pass with proxy kernels $\tilde{P}_s$. These kernels are chosen so that their corresponding pullbacks can be evaluated in closed form:
\begin{equation}
    \tilde{h}_u(x)
    =
    \prod_{v\in\ch{u}} \tilde{m}_v(x),
    u\in\S_0, \qquad
    \tilde{m}_v(x)
    :=
    \begin{cases}
        (\tilde{P}_v \tilde{h}_v)(x), & v\in\S,\\
        p_v(x_v\mid x), & v\in\L.
    \end{cases} \label{eq:guide_message}
\end{equation}
Once $\tilde{h}_s(x)$ is specified for all $s\in\S$, define
\begin{equation}
    P^{\circ}_s(x, \di y)
    :=
    \frac{\tilde{h}_s(y)}{(P_s\tilde{h}_s)(x)}
    P_s(x, \di y)
    \label{eq:guided_transition}
\end{equation}
as the \emph{guided proposal} of the exact smoothing kernel $P^{\star}_s(x, \di y)$. The proxy kernels $\tilde{P}_s$ are used only to build the guide $\tilde{h}_s$; the guided proposal is renormalized using original kernel $P_s$. The ratio between the exact smoothing distribution of the full tree $\P^{\star}$ and the proposal distribution $\P^{\circ}$ induced by $P^{\circ}_s$ can then be evaluated explicitly \citep[Theorem 8]{meulenfrankvander_backwardfilteringforward_2025}:
\begin{equation}
    \frac{\di \P^{\star}}{\di \P^{\circ}}(x_\S)
    =
    \frac{\tilde{h}_\0(x_\0)}{h_\0(x_\0)}
    \prod_{v\in\V^+} \omega_v(x_{\pa{v}}), \qquad
    \omega_v(x)
    =
    \begin{cases}
        (P_v \tilde{h}_v)(x) / (\tilde{P}_v \tilde{h}_v)(x), & v \in \S,\\
        1, & v \in \L.
    \end{cases}
    \label{eq:bffg_weight}
\end{equation}
This ratio can be used as an importance weight in SMC or MCMC-based sampling frameworks. Such a filtering-guiding procedure is referred to as Backward Filtering Forward Guiding (BFFG)\citep{meulenfrankvander_backwardfilteringforward_2025}. The filtering procedure is illustrated in Figure~\ref{fig:nbffg} (b). Directly learning the exact $h$ from prior samples is difficult: descendant evidence may enter at different tree depths and through multiple states, while uninformed prior trajectories rarely visit high-likelihood regions under informative leaves. This motivates using a guided proposal as the base distribution, a point we examine in Section~\ref{sec:exp1_exact_posterior}.

\section{Neural Backward Filtering Forward Guiding}

% \vspace{-5pt}
\subsection{A canonical choice of proxy}
% \vspace{-5pt}
The choice of proxy kernel $\tilde{P}_s$ in BFFG can be nontrivial in general nonlinear models. There is an inherent trade-off between computational efficiency and sampling quality. A simple proxy may lead to closed-form backward messages and an easily sampled guided proposal, but it can also produce high-variance importance weights when $\P^{\circ}$ is far from $\P^{\star}$. While model-specific strategies for proxy selection exist \citep{mider_continuousdiscretesmoothingdiffusions_2021}, we take a different route: in NBFFG, the proxy only needs to provide a cheap global guide, and the remaining mismatch is handled by a neural correction.

For the models considered in this work, we therefore use a canonical proxy that ignores the nonlinear mean $\mu_v$ in Equation~\eqref{eq:discrete_edge} or the drift $b_v$ in Equation~\eqref{eq:continuous_edge}, while retaining a state-independent covariance or diffusion coefficient:
\begin{subequations}
    \label{eq:canonical_proxy}
    \begin{align}
        \text{discrete proxy:}\quad
        &\tilde{P}_v(x, \di y)
        =
        \mathcal{N}(\di y; x, \tilde{\Sigma}_v), \label{eq:discrete_proxy}\\
        \text{continuous proxy:}\quad
        &\di \tilde{Z}_v(t)
        =
        \tilde{\sigma}_v(t) \di \tilde{W}(t),
        \qquad
        \tilde{Z}_v(0)=x. \label{eq:continuous_proxy}
    \end{align}
\end{subequations}
For the continuous proxy, we write $\tilde{P}_{v,t\to T_v}, 0\leq t\leq T_v$ for the proxy transition semigroup from time $t$ to the endpoint $T_v$: $(\tilde{P}_{v,t\to T_v}f)(z)=\mathbb{E}_{\tilde{P}_v}\!\left[f(\tilde{Z}_v(T_v))\mid \tilde{Z}_v(t)=z\right]$. When the leaf-edge observation models are also Gaussian, this proxy yields closed-form guiding messages and closed-form scores for continuous edges.

\begin{proposition}[Gaussian proxy guide]
\label{prop:gaussian_guiding_messages}
Assume Gaussian leaf likelihoods $p_l(x_l\mid x)=\mathcal{N}(x_l;L_lx+\beta_l,R_l)$ and the proxy kernels in Equation~\eqref{eq:canonical_proxy}, with nonsingular information matrices. Then the backward proxy guides remain quadratic:
\[
    \log\tilde{h}_u(x)
    =
    -\frac{1}{2}x^\top\tilde{H}_u x +\tilde{\eta}_u^\top x + \mathrm{const}.
\]
For a continuous edge $(u,v)$, the within-edge guide $\tilde{h}_{v,t}(z):=(\tilde{P}_{v,t\to T_v}\tilde{h}_v)(z)$ is also quadratic, so its score is affine:
\[
    \nabla_z\log\tilde{h}_{v,t}(z)=-\tilde{H}_{v,t}z+\tilde{\eta}_{v,t}.
\]
The information parameters are given by the closed-form backward recursion in Appendix~\ref{app:proofs}.
\end{proposition}

This gives exactly the quantities needed below: in discrete time, the proxy guide yields the base guided proposal through Equation~\eqref{eq:guided_transition}; in continuous time, its affine score serves as part of the SDE control. We do not require $\tilde{h}_v$ or $\tilde{h}_{v,t}$ to be accurate. Their role is to inject coarse descendant information, while the neural component introduced next learns the residual correction.

% \vspace{-5pt}
\subsection{Variational posterior}
% \vspace{-5pt}

We now define the neural correction at the level of Markov kernels. For each unobserved internal vertex $s\in\S$, let $Q_s^\theta(x_{\pa{s}},\di x_s)$ be a variational transition kernel parameterized by $\theta$, absolutely continuous with respect to the prior transition $P_s(x_{\pa{s}},\di x_s)$. We define the joint variational law as the tree-structured measure
\begin{equation}
    \Q^\theta(\di x_\S\mid x_\0) =
    \prod_{s\in\S} Q_s^\theta(x_{\pa{s}},\di x_s).
    \label{eq:variational_tree_law}
\end{equation}
The leaves are not sampled under $\Q^\theta$; they are observed boundary values and enter the objective defined below through the leaf-edge densities $p_l(x_l\mid x_{\pa{l}})$.

The concrete parameterization of $Q_s^\theta$ will differ between discrete and continuous edges. The common principle is that $Q_s^\theta$ is built by correcting the proxy-induced guide rather than by approximating the posterior from scratch. The following objective, however, does not depend on how the correction is parameterized.

\begin{proposition}[Tree variational object]
\label{prop:tree_elbo}
Assume $Q_s^\theta(x,\cdot)\ll P_s(x,\cdot)$ for all $s\in\S$ and $Q^\theta_s$--a.e. $x$. Define the negative ELBO (NELBO)
\begin{equation}
    \J(\theta)
    =
    \mathbb{E}_{\Q^\theta}
    \left[
        \sum_{s\in\S}
        \log
        \frac{\di Q_s^\theta(X_{\pa{s}},\cdot)}
             {\di P_s(X_{\pa{s}},\cdot)}(X_s)
        -
        \sum_{l\in\L}
        \log p_l(x_l\mid X_{\pa{l}})
    \right].
    \label{eq:abstract_elbo}
\end{equation}
Then
\[
    \J(\theta)
    =
    \kld\!\left(
        \Q^\theta(\cdot\mid x_\0)
        \,\middle\Vert\,
        \P^\star(\cdot\mid x_\L,x_\0)
    \right)
    -
    \log Z(x_\L,x_\0).
\]
Thus minimizing $\J(\theta)$ is equivalent to minimizing the KL divergence from $\Q^\theta$ to the exact smoothing posterior $\P^{\star}$.
\end{proposition}

This objective makes the role of the proxy explicit. The target posterior and the NELBO are defined using the true transitions $P_s$ and the true leaf-edge densities. The proxy guide $\tilde{h}$ enters only through the parameterization of $Q_s^\theta$. We then specialize to two different edge dynamics---discrete and continuous---leading to different variational implementations.

% \vspace{-5pt}
\paragraph{Discrete transitions}
For a discrete internal edge $s\in\S$, the proxy guide gives an explicit Gaussian base. Suppose
$P_s(x,\di y)=\mathcal{N}(\di y;\mu_s(x),\Sigma_s(x))$ and
$\log\tilde{h}_s(y)=\mathrm{const}+\tilde{\eta}_s^\top y-\frac{1}{2}y^\top\tilde{H}_s y$.
Then Equation~\eqref{eq:guided_transition} gives $P_s^\circ(x,\di y)=\mathcal{N}(\di y;\mu_s^\circ(x),\Sigma_s^\circ(x))$, with
\begin{equation}
    \Sigma_s^\circ(x) =
    \left(\Sigma_s(x)^{-1}+\tilde{H}_s\right)^{-1},
    \qquad
    \mu_s^\circ(x) =
    \Sigma_s^\circ(x)
    \left(\Sigma_s(x)^{-1}\mu_s(x)+\tilde{\eta}_s\right).
    \label{eq:discrete_guided_gaussian}
\end{equation}
We parameterize $Q_s^\theta(\cdot\mid x)$ as a neural Gaussian mixture that is initialized at this guided base and then corrected by neural residuals. Increasing the number of mixture components lets the variational transition represent skewed or multimodal smoothing structure, while the NELBO only requires sampling from $Q_s^\theta$ and evaluating $\log q_s^\theta-\log p_s$. The exact network parameterization is given in Appendix~\ref{app:edgewise_corrections}.

% \vspace{-5pt}
\paragraph{Continuous transitions}
For a continuous edge, the variational object is a path measure on $\mathcal{C}([0,T_s],\R^d)$. We use the same notation $Q_s^\theta$ for this path law; its endpoint pushforward gives the vertex-level transition appearing in Equation~\eqref{eq:variational_tree_law}. Let $a_s(z)=\sigma_s(z)\sigma_s(z)^\top$. We define a controlled SDE
\begin{equation}
    \di Z_s(t)
    =
    \left[
        b_s(Z_s(t))
        +
        a_s(Z_s(t))g_s^\theta(t,Z_s(t))
    \right]\di t
    +
    \sigma_s(Z_s(t))\di W_t,
    \qquad
    Z_s(0)=x_{\pa{s}},
    \label{eq:variational_sde}
\end{equation}
where $W$ is a $\Q^\theta$-Wiener process. The control decomposes into the fixed proxy score $\nabla_z\log\tilde{h}_{s,t}$ and a neural residual $r_s^\theta$,
\begin{equation}
    g_s^\theta(t,z)
    =
    \nabla_z\log\tilde{h}_{s,t}(z)
    +
    r_s^\theta(t,z)
    =
    -\tilde{H}_{s,t}z+\tilde{\eta}_{s,t}
    +
    r_s^\theta(t,z).
    \label{eq:neural_control}
\end{equation}
Under standard Novikov-type conditions, and assuming $a_s(z)$ is positive definite uniformly in $z$. The Girsanov's theorem gives the edge-wise KL contribution
\begin{equation}
    \kld(Q_s^\theta\Vert P_s)
    =
    \frac{1}{2}
    \mathbb{E}_{Q_s^\theta}
    \left[
        \int_0^{T_s}
        g_s^\theta(t,Z_s(t))^\top
        a_s(Z_s(t))
        g_s^\theta(t,Z_s(t))
        \di t
    \right].
    \label{eq:continuous_edge_kl}
\end{equation}
The affine proxy score steers paths toward descendant observations, and the residual field $r_s^\theta$ learns the nonlinear correction needed by the exact posterior; implementation details are given in Appendix~\ref{app:edgewise_corrections}. Thus, in both regimes, the target posterior and ELBO use the true generative model, while the proxy supplies only a structured base for neural correction.

\subsection{Tree-path subsampling}

The tree NELBO is a sum of local transition and leaf-edge costs. Evaluating every term at each optimization step can be expensive on large trees, especially for continuous edges where each transition requires simulating a controlled path. NBFFG therefore uses topology-level path subsampling: draw a mini-batch of observed leaves $B\subseteq\L$ and simulate only the activated internal vertices $A(B)=\{s\in\S:\exists l\in B,\ s\preceq l\}$, i.e., the union of their root-to-leaf paths; let $A_0(B)=A(B)\cup\{\0\}$. Let $C_s^\theta$ be the internal-edge contribution to the ELBO: in discrete time, $C_s^\theta=\log q_s^\theta(X_s\mid X_{\pa{s}})-\log p_s(X_s\mid X_{\pa{s}})$, and in continuous time it is the control-energy term in Equation~\eqref{eq:continuous_edge_kl}. For an observed leaf, let $C_l=-\log p_l(x_l\mid X_{\pa{l}})$. With inclusion probabilities $\pi_s=\mathbb{P}\{s\in A(B)\}$ and $\rho_l=\mathbb{P}\{\pa{l}\in A_0(B)\}$, we optimize the path loss
\begin{equation}
    \widehat{\J}_{B}(\theta)
    =
    \sum_{s\in A(B)}
    \frac{C_s^\theta}{\pi_s}
    +
    \sum_{\substack{l\in\L\\\pa{l}\in A_0(B)}}
    \frac{C_l}{\rho_l}
    \label{eq:path_subsampled_elbo}
\end{equation}
which is an unbiased estimator of the full-tree objective under the sampled path-union law. The first sum is the standard Horvitz--Thompson correction for sampled internal edges. The second sum is an optional Rao--Blackwellization weighted by the inclusion probability of selected leaf parents: it evaluates all available observed children of each sampled leaf parent, reducing variance without changing the sampled internal path union. Appendix~\ref{app:path_subsampling_details} gives the formal statement, implementation details, and proof.

% \vspace{-5pt}
\subsection{Amortized correction}
% \vspace{-5pt}

The neural correction can be parameterized independently on every edge, but this scales poorly and prevents information sharing across related branches. We instead use shared parameters with edge-specific context:
\[
    Q_s^\theta(\cdot\mid x)
    =
    Q^\theta(\cdot\mid x;\psi_s),
    \qquad s\in\S,
\]
where $\psi_s$ summarizes the root-to-node structural context of edge $s$. In discrete time, the shared network predicts residual mixture parameters around the guided Gaussian base in Equation~\eqref{eq:discrete_guided_gaussian}; in continuous time, it predicts the residual control $r_s^\theta$ in Equation~\eqref{eq:neural_control}. The proxy supplies coarse descendant information, while the amortized neural component learns corrections reused across sampled and unsampled branches. Appendix~\ref{app:edgewise_corrections} gives the explicit constructions of amortized correction networks for discrete and continuous edges, respectively, and Appendix~\ref{app:amortized_correction_details} describes how the structural context $\psi_s$ is implemented.

% \vspace{-5pt}
\section{Experiments}
% \vspace{-5pt}

\subsection{Gaussian tree benchmarks with exact references}
\label{sec:exp1_exact_posterior}

We first use Gaussian tree models where exact smoothing references are available. Benchmarks~1A and~1C use heterogeneous linear-Gaussian edges \(X_v=A_vX_{\pa{v}}+b_v+\epsilon_v\), \(\epsilon_v\sim\mathcal{N}(0,Q)\). Benchmarks~1B and~1D use heterogeneous Ornstein--Uhlenbeck edges \(\di Z_v(t)=B_v(\theta_v-Z_v(t))\di t+\sigma \di W_t\), whose closed-form endpoint transitions also give exact Gaussian tree smoothing references. Benchmarks~1A--1B test proxy misspecification, while 1C--1D isolate path subsampling under the canonical proxy. We use five independent observation instances for the proxy benchmarks and report metrics from Equation~\eqref{eq:exp1_metrics}; full model, topology, and training details are in Appendices~\ref{app:exp1a_details}--\ref{app:exp1d_details}.

For proxy correction, we compare the optimal matching-transition proxy with canonical and sign-perturbed misspecified proxies, and include a no-guidance ablation trained from the prior without proxy messages. Table~\ref{tab:exp1_core} shows that neural correction sharply reduces the relative NELBO gap in the discrete benchmark and gives a consistent reduction in the continuous benchmark. The no-guidance ablation remains worse after correction in both cases, supporting the need for guided proposals. KL diagnostics show the same effect for most misspecified proxies; complete empirical posterior mean and covariance diagnostics are in Tables~\ref{tab:exp1a_proxy_correction} and~\ref{tab:exp1b_proxy_correction}. For continuous OU edges, the corrected optimal-proxy entries are nonzero because the controlled SDE objective is evaluated under a fixed Euler--Maruyama discretization, so they should be read as discretization- and Monte Carlo-limited reference values rather than model misspecification.

\begin{table}[t]
\centering
\caption{Exact-posterior benchmark results. The metrics before (Uncorr.) and after (Corr.) training are reported. Values are mean \(\pm\) std. over five independent runs with different seeds. The proxy column uses \((\tilde A,\tilde b)\) for the discrete LG model and \((\tilde B,\tilde\theta)\) for the OU model. We show only \(\Delta_{\mathrm{rel}}\) and \(\mathrm{KL}_{\mathrm{avg}}\) from Equation~\eqref{eq:exp1_metrics}; complete diagnostics are reported in Tables~\ref{tab:exp1a_proxy_correction} and~\ref{tab:exp1b_proxy_correction}.}
\label{tab:exp1_core}
\scriptsize
\resizebox{\textwidth}{!}{%
\begin{tabular}{@{}c@{\qquad}c@{}}
\textbf{1A: Discrete Linear Gaussian} & \textbf{1B: Continuous Heterogeneous OU} \\[2pt]
\begin{tabular}{@{}lrrrr@{}}
\toprule
& \multicolumn{2}{c}{\(\Delta_{\mathrm{rel}}\downarrow\)}
& \multicolumn{2}{c}{\(\mathrm{KL}_{\mathrm{avg}}\) \(\downarrow\)} \\
\cmidrule(lr){2-3}\cmidrule(l){4-5}
Proxy & Uncorr. & Corr. & Uncorr. & Corr. \\
\midrule
No guidance & \(1.70{\times}10^4\pm7.34{\times}10^3\) & \(96.659\pm15.237\) & \(9.670\pm2.548\) & \(75.495\pm24.010\) \\
\((A,b)\) & \(0.000\pm0.000\) & \(0.002\pm0.001\) & \(0.056\pm0.004\) & \(0.060\pm0.010\) \\
\((I,0)\) \eqref{eq:discrete_proxy} & \(0.913\pm0.230\) & \(0.035\pm0.011\) & \(1.861\pm0.096\) & \(0.118\pm0.004\) \\
\((-A,b)\) & \(1.227\pm0.322\) & \(0.016\pm0.005\) & \(1.677\pm0.602\) & \(0.080\pm0.017\) \\
\((A,-b)\) & \(3.352\pm0.800\) & \(0.046\pm0.015\) & \(6.559\pm0.051\) & \(0.096\pm0.022\) \\
\((-A,-b)\) & \(4.482\pm0.743\) & \(0.078\pm0.040\) & \(6.690\pm1.073\) & \(0.111\pm0.020\) \\
\bottomrule
\end{tabular}%
&
\begin{tabular}{@{}lrrrr@{}}
\toprule
& \multicolumn{2}{c}{\(\Delta_{\mathrm{rel}}\downarrow\)}
& \multicolumn{2}{c}{\(\mathrm{KL}_{\mathrm{avg}}\) \(\downarrow\)} \\
\cmidrule(lr){2-3}\cmidrule(l){4-5}
Proxy & Uncorr. & Corr. & Uncorr. & Corr. \\
\midrule
No guidance & \(29.245\pm17.626\) & \(0.605\pm0.226\) & \(11.904\pm1.922\) & \(0.144\pm0.022\) \\
\((B,\theta)\) & \(0.892\pm0.583\) & \(0.213\pm0.075\) & \(0.261\pm0.051\) & \(0.048\pm0.011\) \\
\((0,0)\) \eqref{eq:continuous_proxy} & \(1.273\pm0.724\) & \(0.262\pm0.104\) & \(0.387\pm0.121\) & \(0.062\pm0.008\) \\
\((DBD,\theta)\) & \(1.294\pm0.648\) & \(0.214\pm0.073\) & \(0.431\pm0.051\) & \(0.045\pm0.004\) \\
\((B,-\theta)\) & \(5.122\pm1.722\) & \(0.244\pm0.092\) & \(1.381\pm0.208\) & \(0.056\pm0.008\) \\
\((DBD,-\theta)\) & \(3.955\pm1.470\) & \(0.264\pm0.111\) & \(1.105\pm0.105\) & \(0.062\pm0.009\) \\
\bottomrule
\end{tabular}%
\end{tabular}%
}
\end{table}

For path subsampling, Benchmarks~1C--1D in Appendices~\ref{app:exp1c_details}--\ref{app:exp1d_details} show the expected accuracy--speed trade-off. Very aggressive leaf mini-batches with few sampled leaves are noisy on the wider discrete tree, while moderate mini-batches retain accuracy with a useful speedup. The continuous OU benchmark is more tolerant of aggressive subsampling and benefits more in wall-clock time because each skipped branch avoids many Euler--Maruyama steps. Rao--Blackwellized terminal sampling behaves as an optional implementation improvement rather than a separate modeling contribution; the full curves and tree diagrams are in Appendices~\ref{app:exp1c_details} and~\ref{app:exp1d_details}.

% \vspace{-5pt}
\subsection{Nonlinear multimodal smoothing}
\label{sec:exp2_nonlinear}
% \vspace{-5pt}

We next test whether the discrete Gaussian-mixture correction can represent posterior multimodality induced by nonlinear hidden dynamics. Benchmark~2A uses a fixed super-root $\0$ and an inferred stochastic root $\rho\in\S$, whose children follow the folded nonlinear Gaussian transition \(X_v=f(X_\rho)+\epsilon_v\), \(f(x)=(x_1^2,x_2^2)\), with all linear Gaussian observation leaves fixed at \((1,1)\). The four sign choices of $X_\rho$ are therefore all plausible. We also deploy a two-dimensional finite grid to evaluate the analytical posterior, which gives a reference with four equiprobable quadrant modes. Details are given in Appendix~\ref{app:exp2a_details}.

Figure~\ref{fig:exp2a_root_modes} shows that the raw guide and the single-Gaussian correction (by setting $K=1$ in Equation~\eqref{eq:app_discrete_neural_mixture}) collapse to one quadrant, while NBFFG with four mixture components ($K=4$) recovers the four-mode root posterior in the representative run shown here. An unguided Gaussian mixture is less stable across independent runs, where the learned variational deviates far from the reference; full diagnostics are reported in Table~\ref{tab:exp2a_multimodal}. Overall, this isolates the role of the guided Gaussian-mixture correction in representing multimodal discrete smoothing posteriors.

\begin{figure}[t]
    \centering
    \includegraphics[width=\textwidth]{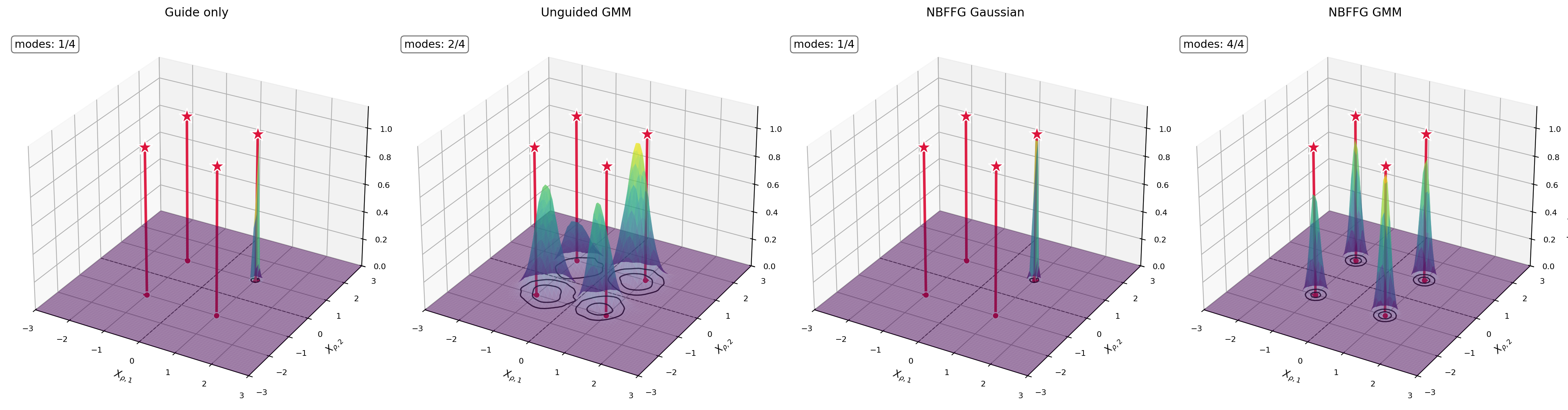}
    \caption{Inferred-root posterior samples for Benchmark~2A. The folded nonlinear transition creates four equiprobable posterior modes for $X_\rho$ centered at four red pillars, and the mesh surfaces represent empirical distributions evaluated on a finite grid over $[-3,3]^2$. The contours on the $xy$--plane indicate the projections of the distributions. NBFFG with a Gaussian-mixture correction $(K=4)$ recovers the multimodal posterior, while the raw guide and single-Gaussian ($K=1$) correction collapse to one quadrant.}
    % \vspace{-15pt}
    \label{fig:exp2a_root_modes}
\end{figure}

Benchmark~2B moves to continuous nonlinear smoothing. Each hidden edge follows the double-well diffusion \(\di Z_v(t)=-4\alpha Z_v(t)(Z_v(t)^2-1)\di t+\sigma \di W_t\), with Gaussian observation leaves attached to terminal latent vertices. We use two observation regimes. The early-commitment regime requires the two main branches to separate near the root, while the bimodal regime leaves competing explanations for the branch assignments. Because no exact reference is available, we treat this as a stress test and compare proposal efficiency and path diagnostics rather than reference errors; details are given in Appendix~\ref{app:exp2b_details}.

Figure~\ref{fig:exp2b_paths} shows the resulting posterior path samples. In the early-commitment regime, NBFFG propagates leaf evidence into an early branch separation and improves normalized ESS from \(0.273\pm0.009\) for guide only and \(0.271\pm0.016\) for guided SMC to \(0.529\pm0.103\). The bimodal regime is substantially harder: the preconditioned Crank--Nicolson (guided pCN) collapses to one path-space mode and guided SMC covers modes only with severe path degeneracy, while NBFFG is the only method that combines mode coverage with visibly diverse independent paths. Its normalized ESS remains low but is the largest among the compared proposals, \(0.0076\pm0.0058\), indicating that robust multimodal path-space correction remains challenging.

\begin{figure}[t]
    \centering
    \includegraphics[width=\textwidth]{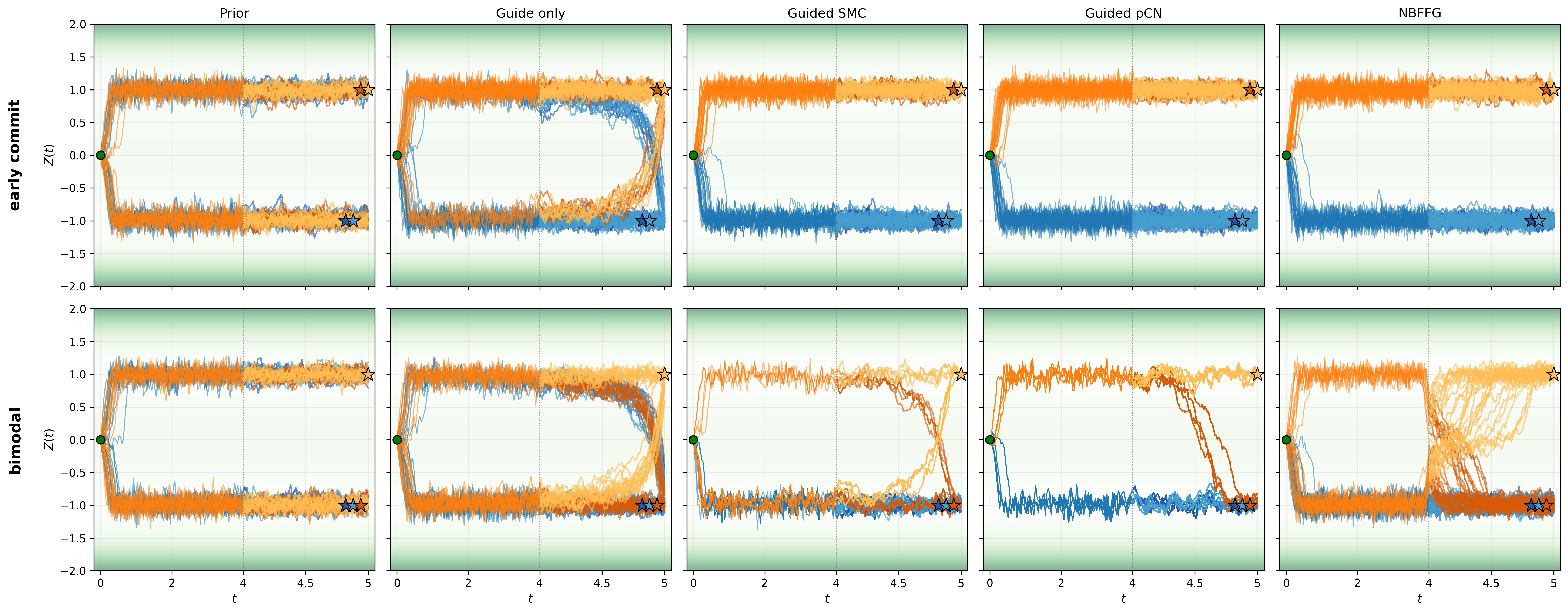}
    \caption{Posterior path samples for Benchmark~2B. The composite figure shows early-commitment and bimodal observation regimes. Paths with different colors indicate the corresponding edges. Observations are marked by stars in the colors of their corresponding edges. Each panel displays 30 paths. In the early-commit regime, Guided pCN/SMC and NBFFG give consistent results; in the bimodal regime, guided pCN collapses to one mode and guided SMC suffers path degeneracy, while NBFFG better preserves both mode coverage and path diversity. The stand-alone guided proposal fails to recover the posterior dynamics in both regimes.}
    % \vspace{-15pt}
    \label{fig:exp2b_paths}
\end{figure}

% \vspace{-5pt}
\subsection{Phylogenetic butterfly shape inference}
\label{sec:exp3_butterfly}
% \vspace{-5pt}

As a real-world application, we apply NBFFG to phylogenetic shape inference on butterfly wing outlines. We add a fixed super-root $\0$ and infer the original phylogenetic root $\rho\in\S$. The latent state is a high-dimensional landmark representation of wing shape, and observed extant species are attached as leaf observations on the phylogeny. Unlike the previous benchmarks, no exact posterior is available; the goal is therefore to assess whether NBFFG can produce observation-conditioned ancestral reconstructions on a realistic nonlinear tree-structured problem. Details are given in Appendix~\ref{app:exp3_butterfly_details}.

Figure~\ref{fig:butterfly_phylogeny} shows posterior mean shapes sampled by NBFFG at internal ancestral nodes. The dashed gray curves show the unconditional forward-prior mean, while the blue curves show the NBFFG posterior mean conditioned on observed leaf shapes. The posterior reconstructions are pulled toward descendant observations across the tree, including at the inferred most recent common ancestor, whereas the unconditional prior does not incorporate the leaf evidence. This qualitative result demonstrates that NBFFG scales beyond synthetic benchmarks to high-dimensional nonlinear phylogenetic smoothing.

\begin{figure}[t]
    \centering
    \includegraphics[width=\textwidth]{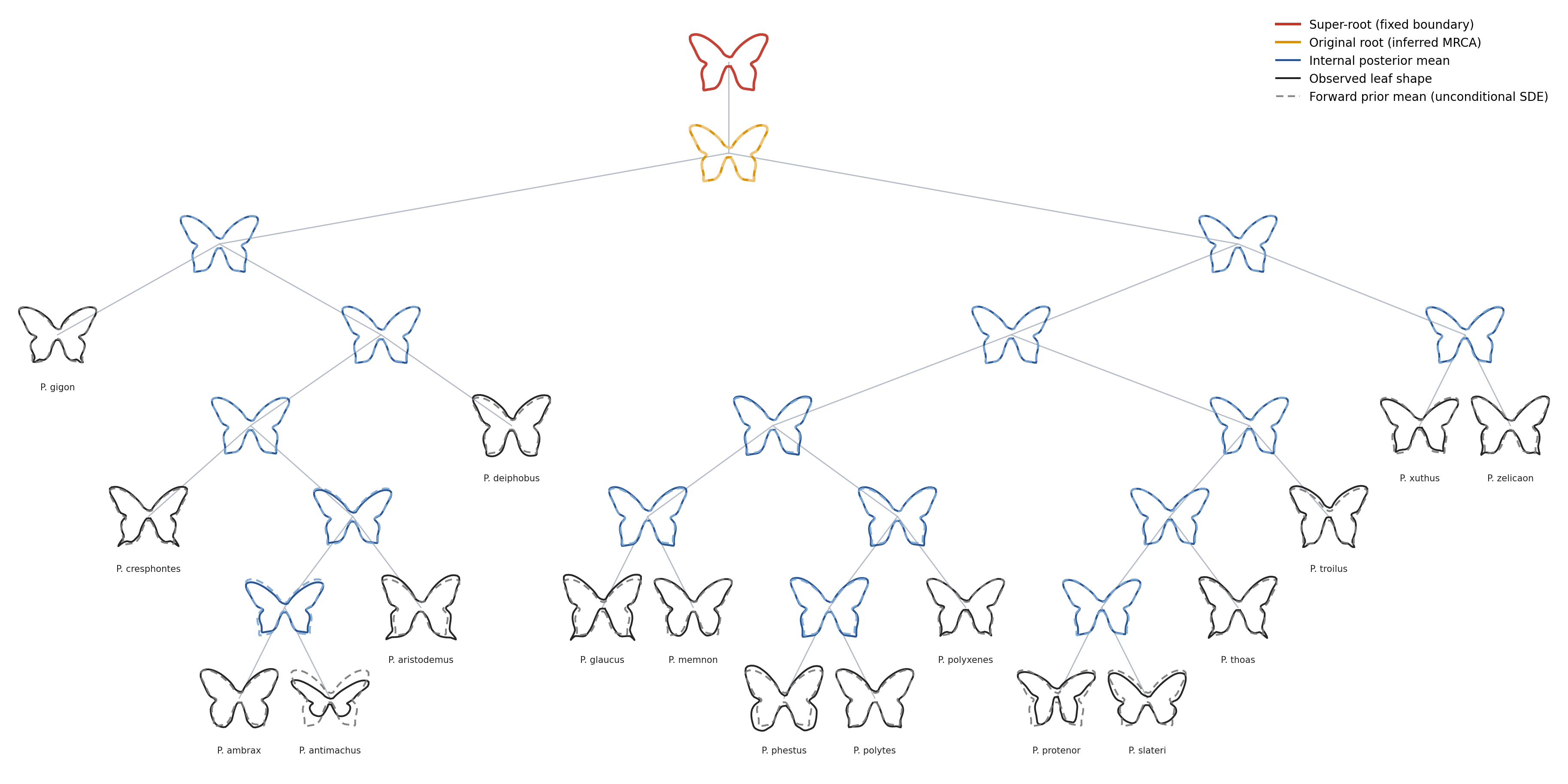}
    \caption{Phylogenetic butterfly shape inference. Observed extant species are shown in black, NBFFG posterior means of 200 independent samples at internal nodes in blue, and unconditional forward-prior means of 200 independent samples in dashed gray. NBFFG produces ancestral shape reconstructions conditioned on the observed leaves, while the prior mean ignores descendant evidence.}
    % \vspace{-15pt}
    \label{fig:butterfly_phylogeny}
\end{figure}

% \vspace{-5pt}
\section{Related Work}
% \vspace{-10pt}
\paragraph{State-space smoothing and message passing}
For linear-Gaussian state-space models, the Rauch--Tung--Striebel smoother gives exact smoothing recursions \citep{rauch_maximumlikelihood_1965}; see \citet[Chapter~12]{sarkka_bayesianfilteringsmoothing_2023} for a modern treatment. Forward filtering--backward sampling was introduced for Bayesian state-space inference by \citet{carter1994gibbs} and \citet{fruhwirth1994data}. On trees, the same backward-then-forward structure appears as Pearl's belief propagation \citep{pearl1988probabilistic}, the sum-product algorithm on factor graphs \citep{kschischang_factorgraphs_2001}, and Felsenstein's pruning recursion in phylogenetics \citep{felsenstein_evolutionarytreesdna_1981}. Linear-Gaussian and pairwise Markov tree variants have also been studied in multiscale estimation and signal processing \citep{chou_multiscalerecursiveestimation_1994,desbouvries_kalmanfilteringpairwise_2006}.

% \vspace{-10pt}
\paragraph{Particle smoothing, twisting, and variational objectives}
For nonlinear or non-Gaussian state-space models, particle smoothers provide a general simulation-based alternative. Two-filter smoothing \citep{briers_smoothingalgorithmsstate_2010} and twisted particle filters \citep{whiteley_twistedparticlefilters_2014} make explicit the role of future information in reducing weight degeneracy. Subsequent work approximates optimal twists iteratively \citep{guarniero_iteratedauxiliaryparticle_2017}, casts twisting as stochastic control \citep{heng_controlledsequentialmonte_2020}, or learns proposals and Doob-type transformations \citep{paige2016inference,chopin_computationaldoobhtransforms_2023}. Related variational SMC objectives include FIVO and VSMC \citep{maddison_filteringvariationalobjectives_2017,naesseth_variationalsequentialmonte_2018}, while SI$\chi$O and NAS-$\chi$ learn smoothing-aware twists within reweighted wake-sleep objectives \citep{lawson_sihosmoothinginference_2022,lawson_nashneuraladaptive_2023}.

% \vspace{-10pt}
\paragraph{Diffusion bridges and guided smoothing}
Diffusion smoothing is often formulated as a bridge problem. Classical guided proposals add an auxiliary drift that steers paths toward future observations and then correct the mismatch by a change of measure \citep{delyon_simulationconditioneddiffusions_2006,schauer_guidedproposalssimulating_2017}; related continuous-discrete smoothing and control perspectives are developed by \citet{mider_continuousdiscretesmoothingdiffusions_2021} and \citet{eklund_guidedsmoothingcontrol_2026}. BFFG extends this backward-filtering forward-guiding principle to directed trees and DAGs \citep{meulenfrankvander_backwardfilteringforward_2025}. Recent learning-based bridge samplers use variational diffusion objectives \citep{archambeau_variationalinference_2007}, score matching \citep{heng_simulatingdiffusionbridges_2025,baker_scorematchingbridges_2025}, control-consistency losses \citep{howard_controlconsistencylosses_2025}, or neural guided controls \citep{yang_neuralguideddiffusion_2025}. NBFFG builds on BFFG and neural guided diffusion bridges, but differs by using a deliberately simple proxy guide, and learns a neural correction over the whole tree.

% \vspace{-10pt}
\section{Conclusion}
% \vspace{-5pt}

We introduced Neural Backward Filtering Forward Guiding, a variational framework for smoothing on tree-structured stochastic processes. The central idea is to separate global information propagation from local posterior correction: a tractable proxy model provides closed-form backward messages and a guided proposal, while a neural correction repairs the mismatch between this guide and the true posterior. This gives a unified construction for discrete tree transitions and continuous diffusion edges, with Gaussian-mixture corrections in discrete time and controlled SDEs in continuous time.

Experiments show that neural correction substantially improves deliberately misspecified guides in exact-posterior LG and OU benchmarks, while guidance remains important compared with training from the prior. The same construction represents nonlinear multimodal smoothing posteriors and scales through path-level subsampling and amortized correction, as illustrated by high-dimensional phylogenetic butterfly shape inference.

The current construction also leaves clear extensions. Path subsampling avoids full-tree updates, but training still backpropagates through sampled root-to-leaf paths; a Bellman-style actor-critic formulation could further localize learning through edge-wise corrections and propagated value information. Amortization is inductive only to a degree: new observations on a fixed topology can be handled by recomputing filtering information, while transfer across unseen topologies calls for correction networks with stronger topology-aware representations.

% \newpage
\bibliographystyle{unsrtnat}
\bibliography{reference.bib}

@inproceedings{archambeau_variationalinference_2007,
  author    = {Archambeau, C{\'e}dric and Opper, Manfred and Shen, Yuan and Cornford, Dan and Shawe-Taylor, John},
  title     = {Variational Inference for Diffusion Processes},
  booktitle = {Advances in Neural Information Processing Systems 20},
  year      = {2007},
  url       = {https://papers.nips.cc/paper/3282-variational-inference-for-diffusion-processes}
}

@inproceedings{baker_scorematchingbridges_2025,
  author    = {Baker, Elizabeth L. and Schauer, Moritz and Sommer, Stefan},
  title     = {Score Matching for Bridges without Learning Time-Reversals},
  booktitle = {Proceedings of the 28th International Conference on Artificial Intelligence and Statistics},
  series    = {Proceedings of Machine Learning Research},
  volume    = {258},
  pages     = {775--783},
  year      = {2025},
  publisher = {PMLR},
  url       = {https://proceedings.mlr.press/v258/baker25a.html},
  doi       = {10.48550/arXiv.2407.15455}
}

@book{bovier_gaussian_2017,
  author    = {Bovier, Anton},
  title     = {Gaussian Processes on Trees: From Spin Glasses to Branching Brownian Motion},
  publisher = {Cambridge University Press},
  year      = {2017},
  doi       = {10.1017/9781316675779}
}

@article{briers_smoothingalgorithmsstate_2010,
  author  = {Briers, Mark and Doucet, Arnaud and Maskell, Simon},
  title   = {Smoothing Algorithms for State-Space Models},
  journal = {Annals of the Institute of Statistical Mathematics},
  volume  = {62},
  number  = {1},
  pages   = {61--89},
  year    = {2010},
  doi     = {10.1007/s10463-009-0236-2}
}

@article{butler_phylogenetic_2004,
  author  = {Butler, Marguerite A. and King, Aaron A.},
  title   = {Phylogenetic Comparative Analysis: A Modeling Approach for Adaptive Evolution},
  journal = {The American Naturalist},
  volume  = {164},
  number  = {6},
  pages   = {683--695},
  year    = {2004},
  doi     = {10.1086/426002}
}

@article{carter1994gibbs,
  author  = {Carter, Chris K. and Kohn, Robert},
  title   = {On Gibbs Sampling for State Space Models},
  journal = {Biometrika},
  volume  = {81},
  number  = {3},
  pages   = {541--553},
  year    = {1994},
  doi     = {10.1093/biomet/81.3.541}
}

@inproceedings{chopin_computationaldoobhtransforms_2023,
  author    = {Chopin, Nicolas and Fulop, Andras and Heng, Jeremy and Thiery, Alexandre H.},
  title     = {Computational Doob h-Transforms for Online Filtering of Discretely Observed Diffusions},
  booktitle = {Proceedings of the 40th International Conference on Machine Learning},
  series    = {Proceedings of Machine Learning Research},
  volume    = {202},
  pages     = {5904--5923},
  year      = {2023},
  publisher = {PMLR},
  url       = {https://proceedings.mlr.press/v202/chopin23a.html}
}

@article{chou_multiscalerecursiveestimation_1994,
  author  = {Chou, K. C. and Willsky, A. S. and Benveniste, A.},
  title   = {Multiscale Recursive Estimation, Data Fusion, and Regularization},
  journal = {IEEE Transactions on Automatic Control},
  volume  = {39},
  number  = {3},
  pages   = {464--478},
  year    = {1994},
  doi     = {10.1109/9.280746}
}

@article{delyon_simulationconditioneddiffusions_2006,
  author  = {Delyon, Bernard and Hu, Ying},
  title   = {Simulation of Conditioned Diffusion and Application to Parameter Estimation},
  journal = {Stochastic Processes and their Applications},
  volume  = {116},
  number  = {11},
  pages   = {1660--1675},
  year    = {2006},
  doi     = {10.1016/j.spa.2006.04.004}
}

@article{desbouvries_kalmanfilteringpairwise_2006,
  author  = {Desbouvries, Fran{\c c}ois and Lecomte, Jean and Pieczynski, Wojciech},
  title   = {Kalman Filtering in Pairwise Markov Trees},
  journal = {Signal Processing},
  volume  = {86},
  number  = {5},
  pages   = {1049--1054},
  year    = {2006},
  doi     = {10.1016/j.sigpro.2005.07.026}
}

@article{eklund_guidedsmoothingcontrol_2026,
  author  = {Eklund, Oskar and Lang, Annika and Schauer, Moritz},
  title   = {Guided Smoothing and Control for Diffusion Processes},
  journal = {Stochastic Processes and their Applications},
  volume  = {192},
  pages   = {104806},
  year    = {2026},
  doi     = {10.1016/j.spa.2025.104806}
}

@article{felsenstein_evolutionarytreesdna_1981,
  author  = {Felsenstein, Joseph},
  title   = {Evolutionary Trees from DNA Sequences: A Maximum Likelihood Approach},
  journal = {Journal of Molecular Evolution},
  volume  = {17},
  number  = {6},
  pages   = {368--376},
  year    = {1981},
  doi     = {10.1007/BF01734359}
}

@article{fruhwirth1994data,
  author  = {Fr{\"u}hwirth-Schnatter, Sylvia},
  title   = {Data Augmentation and Dynamic Linear Models},
  journal = {Journal of Time Series Analysis},
  volume  = {15},
  number  = {2},
  pages   = {183--202},
  year    = {1994},
  doi     = {10.1111/j.1467-9892.1994.tb00184.x}
}

@article{guarniero_iteratedauxiliaryparticle_2017,
  author  = {Guarniero, Pieralberto and Johansen, Adam M. and Lee, Anthony},
  title   = {The Iterated Auxiliary Particle Filter},
  journal = {Journal of the American Statistical Association},
  volume  = {112},
  number  = {520},
  pages   = {1636--1647},
  year    = {2017},
  doi     = {10.1080/01621459.2016.1222291}
}

@article{heng_controlledsequentialmonte_2020,
  author  = {Heng, Jeremy and Bishop, Adrian N. and Deligiannidis, George and Doucet, Arnaud},
  title   = {Controlled Sequential Monte Carlo},
  journal = {The Annals of Statistics},
  volume  = {48},
  number  = {5},
  pages   = {2904--2929},
  year    = {2020},
  doi     = {10.1214/19-AOS1914}
}

@article{heng_simulatingdiffusionbridges_2025,
  author  = {Heng, Jeremy and De Bortoli, Valentin and Doucet, Arnaud and Thornton, James},
  title   = {Simulating Diffusion Bridges with Score Matching},
  journal = {Biometrika},
  volume  = {112},
  number  = {4},
  pages   = {asaf048},
  year    = {2025},
  doi     = {10.1093/biomet/asaf048}
}

@inproceedings{howard_controlconsistencylosses_2025,
  author    = {Howard, Samuel and N{\"u}sken, Nikolas and Pidstrigach, Jakiw},
  title     = {Control Consistency Losses for Diffusion Bridges},
  booktitle = {Frontiers in Probabilistic Inference: Sampling Meets Learning, NeurIPS 2025 Workshop},
  year      = {2025},
  eprint    = {2512.05070},
  archivePrefix = {arXiv},
  primaryClass  = {stat.ML},
  doi       = {10.48550/arXiv.2512.05070}
}

@article{huelsenbeck_stochastic_2003,
  author  = {Huelsenbeck, John P. and Nielsen, Rasmus and Bollback, Jonathan P.},
  title   = {Stochastic Mapping of Morphological Characters},
  journal = {Systematic Biology},
  volume  = {52},
  number  = {2},
  pages   = {131--158},
  year    = {2003},
  doi     = {10.1080/10635150390192780}
}

@inproceedings{kempe_maximizing_2003,
  author    = {Kempe, David and Kleinberg, Jon and Tardos, {\'E}va},
  title     = {Maximizing the Spread of Influence through a Social Network},
  booktitle = {Proceedings of the Ninth ACM SIGKDD International Conference on Knowledge Discovery and Data Mining},
  pages     = {137--146},
  year      = {2003},
  publisher = {ACM},
  doi       = {10.1145/956750.956769}
}

@article{kschischang_factorgraphs_2001,
  author  = {Kschischang, Frank R. and Frey, Brendan J. and Loeliger, Hans-Andrea},
  title   = {Factor Graphs and the Sum-Product Algorithm},
  journal = {IEEE Transactions on Information Theory},
  volume  = {47},
  number  = {2},
  pages   = {498--519},
  year    = {2001},
  doi     = {10.1109/18.910572}
}

@inproceedings{lawson_nashneuraladaptive_2023,
  author    = {Lawson, Dieterich and Li, Michael Y. and Linderman, Scott W.},
  title     = {{NAS-$\chi$}: Neural Adaptive Smoothing via Twisting},
  booktitle = {Advances in Neural Information Processing Systems 36},
  pages     = {8602--8633},
  year      = {2023},
  url       = {https://proceedings.neurips.cc/paper_files/paper/2023/hash/1b3d005a2cb0e71e698e0b13ac657473-Abstract-Conference.html}
}

@inproceedings{lawson_sihosmoothinginference_2022,
  author    = {Lawson, Dieterich and Ravent{\'o}s, Allan and Warrington, Andrew and Linderman, Scott W.},
  title     = {{SI$\chi$O}: Smoothing Inference with Twisted Objectives},
  booktitle = {Advances in Neural Information Processing Systems 35},
  pages     = {38844--38858},
  year      = {2022},
  url       = {https://papers.neurips.cc/paper_files/paper/2022/hash/fddc79681b2df2734c01444f9bc2a17e-Abstract-Conference.html}
}

@inproceedings{maddison_filteringvariationalobjectives_2017,
  author    = {Maddison, Chris J. and Lawson, Dieterich and Tucker, George and Heess, Nicolas and Norouzi, Mohammad and Mnih, Andriy and Doucet, Arnaud and Teh, Yee Whye},
  title     = {Filtering Variational Objectives},
  booktitle = {Advances in Neural Information Processing Systems 30},
  year      = {2017},
  url       = {https://papers.neurips.cc/paper/7235-filtering-variational-objectives}
}

@article{meulenfrankvander_backwardfilteringforward_2025,
  author  = {{van der Meulen}, Frank H. and Sommer, Stefan},
  title   = {Backward Filtering Forward Guiding},
  journal = {Journal of Machine Learning Research},
  volume  = {26},
  number  = {281},
  pages   = {1--51},
  year    = {2025},
  url     = {https://jmlr.org/papers/v26/25-1130.html}
}

@article{mider_continuousdiscretesmoothingdiffusions_2021,
  author  = {Mider, Marcin and Schauer, Moritz and {van der Meulen}, Frank H.},
  title   = {Continuous-Discrete Smoothing of Diffusions},
  journal = {Electronic Journal of Statistics},
  volume  = {15},
  number  = {2},
  pages   = {4294--4342},
  year    = {2021},
  doi     = {10.1214/21-EJS1894}
}

@inproceedings{naesseth_variationalsequentialmonte_2018,
  author    = {Naesseth, Christian A. and Linderman, Scott W. and Ranganath, Rajesh and Blei, David M.},
  title     = {Variational Sequential Monte Carlo},
  booktitle = {Proceedings of the Twenty-First International Conference on Artificial Intelligence and Statistics},
  series    = {Proceedings of Machine Learning Research},
  volume    = {84},
  pages     = {968--977},
  year      = {2018},
  publisher = {PMLR},
  url       = {https://proceedings.mlr.press/v84/naesseth18a.html}
}

@inproceedings{paige2016inference,
  author    = {Paige, Brooks and Wood, Frank},
  title     = {Inference Networks for Sequential Monte Carlo in Graphical Models},
  booktitle = {Proceedings of the 33rd International Conference on Machine Learning},
  series    = {Proceedings of Machine Learning Research},
  volume    = {48},
  pages     = {3040--3049},
  year      = {2016},
  publisher = {PMLR},
  url       = {https://proceedings.mlr.press/v48/paige16.html}
}

@book{pearl1988probabilistic,
  author    = {Pearl, Judea},
  title     = {Probabilistic Reasoning in Intelligent Systems: Networks of Plausible Inference},
  publisher = {Morgan Kaufmann},
  year      = {1988}
}

@article{rauch_maximumlikelihood_1965,
  author  = {Rauch, H. E. and Tung, F. and Striebel, C. T.},
  title   = {Maximum Likelihood Estimates of Linear Dynamic Systems},
  journal = {AIAA Journal},
  volume  = {3},
  number  = {8},
  pages   = {1445--1450},
  year    = {1965},
  doi     = {10.2514/3.3166}
}

@book{sarkka_bayesianfilteringsmoothing_2023,
  author    = {S{\"a}rkk{\"a}, Simo and Svensson, Lennart},
  title     = {Bayesian Filtering and Smoothing},
  edition   = {2},
  publisher = {Cambridge University Press},
  year      = {2023},
  doi       = {10.1017/9781108917407}
}

@article{schauer_guidedproposalssimulating_2017,
  author  = {Schauer, Moritz and {van der Meulen}, Frank H. and {van Zanten}, Harry},
  title   = {Guided Proposals for Simulating Multi-Dimensional Diffusion Bridges},
  journal = {Bernoulli},
  volume  = {23},
  number  = {4A},
  pages   = {2917--2950},
  year    = {2017},
  doi     = {10.3150/16-BEJ833}
}

@misc{sommer_stochasticsshapeskunitaflows_2025,
  author        = {Sommer, Stefan and Yang, Gefan and Baker, Elizabeth Louise},
  title         = {Stochastics of Shapes and Kunita Flows},
  year          = {2025},
  eprint        = {2512.11676},
  archivePrefix = {arXiv},
  primaryClass  = {math.PR},
  url           = {https://arxiv.org/abs/2512.11676}
}

@article{stroustrup_stochastic_2025,
  author  = {Stroustrup, Sofia and Pedersen, Morten Akh{\o}j and {van der Meulen}, Frank and Sommer, Stefan and Nielsen, Rasmus},
  title   = {Stochastic Phylogenetic Models of Shape},
  journal = {bioRxiv},
  year    = {2025},
  doi     = {10.1101/2025.04.03.646616},
  url     = {https://www.biorxiv.org/content/10.1101/2025.04.03.646616v1}
}

@article{whiteley_twistedparticlefilters_2014,
  author  = {Whiteley, Nick and Lee, Anthony},
  title   = {Twisted Particle Filters},
  journal = {The Annals of Statistics},
  volume  = {42},
  number  = {1},
  pages   = {115--141},
  year    = {2014},
  doi     = {10.1214/13-AOS1167}
}

@inproceedings{yang_neuralguideddiffusion_2025,
  author    = {Yang, Gefan and {van der Meulen}, Frank H. and Sommer, Stefan},
  title     = {Neural Guided Diffusion Bridges},
  booktitle = {Proceedings of the 42nd International Conference on Machine Learning},
  series    = {Proceedings of Machine Learning Research},
  volume    = {267},
  pages     = {71210--71230},
  year      = {2025},
  publisher = {PMLR},
  url       = {https://openreview.net/forum?id=4LClOWTAth}
}

@inproceedings{perez2018film,
  title={Film: Visual reasoning with a general conditioning layer},
  author={Perez, Ethan and Strub, Florian and De Vries, Harm and Dumoulin, Vincent and Courville, Aaron},
  booktitle={Proceedings of the AAAI conference on artificial intelligence},
  volume={32},
  year={2018}
}
%%%%%%%%%%%%%%%%%%%%%%%%%%%%%%%%%%%%%%%%%%%%%%%%%%%%%%%%%%%%

\newpage
\appendix

\section{Additional Problem Setup Details}
\label{app:problem_setup_details}

This section gives the more explicit version of the setup underlying Equation~\eqref{eq:prior}.

\subsection{Markov Kernels as Edge Transitions}
We start with a brief recap of Markov kernels, focusing on their action on functions.

Let $(\X, \fB(\X)), (\Y, \fB(\Y))$ be Borel measurable spaces. A Markov kernel $P:\X\times\fB(\Y)\to[0,1]$ is a map such that:
\begin{itemize}
    \item Given $B\in\fB(\Y)$, the map $x\mapsto P(x, B)$ is $\fB(\X)$-measurable for $x\in \X$;
    \item Given $x\in\X$, the map $B\mapsto P(x, B)$ is a probability measure on $(\Y, \fB(\Y))$.
\end{itemize} 
We call $(\X, \fB(\X))$ and $(\Y, \fB(\Y))$ the ``source'' and ``target'' of $P$, and denote $P:(\X, \fB(\X)) \rightharpoonup (\Y, \fB(\Y))$. For a bounded measurable function $h$ on $(\Y, \fB(\Y))$, the pullback of $h$ through $P$ is
\begin{equation}
    (P h)(x)
    :=
    \int_{\Y} h(y)P(x,\di y).
    \label{eq:kernel_pullback}
\end{equation}

For an edge in the tree, we use an edge-wise Markov kernel $P_v: (\X_{\pa{v}}, \fB(\X_{\pa{v}})) \rightharpoonup (\X_v, \fB(\X_v))$ to describe the conditional law of $X_v$ given $X_{\pa{v}}=x$:
\[
    P_v(x,\di y)
    \equiv
    \P(X_v\in\di y\mid X_{\pa{v}}=x).
\]
When a density $p_v$ of $P_v$ exists with respect to a reference measure $\lambda_v$, we write
\[
    P_v(x,\di y) = p_v(y\mid x)\lambda_v(\di y) .
\]
For a leaf $l\in\L$, the edge kernel is still the generative conditional law of the child given its parent. If the leaf-edge kernel admits a density $p_l$ with respect to a reference measure $\nu_l$, then
\[
    P_l(x,\di y)
    =
    p_l(y\mid x)\nu_l(\di y),
\]
and the smoothing problem conditions on the realized child value $X_l=x_l$. This contributes the factor $p_l(x_l\mid x_{\pa{l}})$ to the posterior over the unobserved internal vertices.

\subsection{Diffusion Paths as Continuous Transitions}

The same Markov-kernel notation applies to continuous-time edges, but there are two useful levels of description. At the tree level, an edge still defines a vertex-to-vertex kernel $P_v(x,\di y)$ from the parent state to the child state. At the path level, this kernel is induced by a conditional diffusion path law. In the continuous case, the vertex-level kernels can be naturally induced by the path-level kernels.

For a continuous edge, conditional on $X_{\pa{v}}=x$, let $\mathbb{P}_v^x$ denote the law on the path space $\mathcal{C}_v:=\mathcal{C}([0,T_v],\R^d)$ of the solution $Z_v$ to the SDE
\begin{equation}
    \di Z_v(t)
    =
    b_v(Z_v(t))\di t
    +
    \sigma_v(Z_v(t))\di W(t),
    \qquad
    Z_v(0)=x.
    \label{eq:app_continuous_prior}
\end{equation}
We assume $b_v$ and $\sigma_v$ satisfy standard regularity conditions ensuring existence and uniqueness of $Z_v$. Define the endpoint evaluation map $e_{T_v}:\mathcal{C}_v\to\R^d$ by $e_{T_v}(z)=z(T_v)$. With this map, we can explicitly define the vertex-level Markov kernel $P_v$ as the pushforward of the path law $\P^x_v$:
\begin{equation}
    P_v(x,A)
    :=
    \mathbb{P}_v^x(e_{T_v}^{-1}(A))
    =
    \mathbb{P}_v^x\{z\in\mathcal{C}_v:z(T_v)\in A\},
    \qquad A\in\fB(\R^d).
    \label{eq:endpoint_kernel_pushforward}
\end{equation}
Equivalently, with $X_v=Z_v(T_v)$,
\[
    P_v(x,A)
    =
    \mathbb{P}_v^x(Z_v(T_v)\in A).
\]
Thus, we can fit continuous edges within the same vertex-level Markov-kernel framework as discrete edges, while retaining an underlying path-space law when pathwise quantities are needed.

The pullback of an endpoint-valued function $h:\R^d\to\R$ through the continuous edge is therefore
\[
    (P_vh)(x)
    :=
    \mathbb{E}_{\mathbb{P}_v^x}\!\left[h(Z_v(T_v))\right].
\]
It is also useful to pull endpoint functions back to intermediate times. For $0\leq t\leq T_v$, define the time-inhomogeneous backward semigroup
\begin{equation}
    (P_{v,t\to T_v}h)(z)
    :=
    \mathbb{E}\!\left[h(Z_v(T_v))\mid Z_v(t)=z\right].
    \label{eq:continuous_backward_semigroup}
\end{equation}
At $t=0$, this recovers the vertex-level pullback, $(P_{v,0\to T_v}h)(x)=(P_vh)(x)$. This is the continuous-time analogue of the discrete pullback in Equation~\eqref{eq:kernel_pullback}.

\section{Additional Method Details}
\label{app:method_details}

\subsection{Edge-Wise Neural Corrections}
\label{app:edgewise_corrections}

For a discrete Gaussian edge, we start with constructing the guided Gaussian base from the true local prior and the proxy message. If $P_s(x,\di y)=\mathcal{N}(\di y;\mu_s(x),\Sigma_s(x))$ and $\log\tilde{h}_s(y)=\mathrm{const}+\tilde{\eta}_s^\top y-\frac{1}{2}y^\top\tilde{H}_sy$ from Proposition~\ref{prop:gaussian_guiding_messages}, by the definition of guided proposal Equation~\eqref{eq:guided_transition}, the guided proposal kernel is again Gaussian: $\mathcal{N}(\di y; \mu^{\circ}_s(x), \Sigma^{\circ}_s(x))$, with the parameters
\begin{equation}
    \begin{aligned}
        \Sigma_s^\circ(x)&=H_s^\circ(x)^{-1}, & \mu_s^\circ(x)&=\Sigma_s^\circ(x)\eta_s^\circ(x), \\
        H_s^\circ(x)&=\Sigma_s(x)^{-1}+\tilde{H}_s, & \eta_s^\circ(x)&=\Sigma_s(x)^{-1}\mu_s(x)+\tilde{\eta}_s.
    \end{aligned}
    \label{eq:app_discrete_base_info}
\end{equation}
The neural correction is applied to this base, not to the prior from scratch. Let $\psi_s$ denote the path context from Appendix~\ref{app:amortized_correction_details}. The Gaussian-mixture network takes
\begin{equation}
    \zeta_s(x)=\left(x,\psi_s,\mu_s^\circ(x),\operatorname{vec}\Sigma_s^\circ(x)\right)
\end{equation}
and outputs logits $a_{s,k}^\theta(x) = a_{s,k}^\theta(\zeta_s(x))$, residual means $\Delta\mu_{s,k}^\theta(x) = \Delta\mu_{s,k}^\theta(\zeta_s(x))$, and lower-triangular factors $M_{s,k}^\theta(x) = M_{s,k}^\theta(\zeta_s(x))$ with positive diagonal. If $L_s^\circ(x)$ is the Cholesky factor of $\Sigma_s^\circ(x)$, the variational transition is
\begin{equation}
    \begin{aligned}
        q_s^\theta(y\mid x) &=\sum_{k=1}^{K}\operatorname{softmax}(a_s^\theta(x))_k\,\mathcal{N}\!\left(y;\mu^{\theta}_{s,k}(x),\Sigma^{\theta}_{s,k}(x)\right). \\
        \mu^{\theta}_{s,k}(x) &= \mu_s^\circ(x)+\Delta\mu_{s,k}^\theta(x),\qquad \Sigma^{\theta}_{s,k}(x) = L_s^\circ(x)M_{s,k}^\theta(x)M_{s,k}^\theta(x)^\top L_s^\circ(x)^\top
    \end{aligned}
    \label{eq:app_discrete_neural_mixture}
\end{equation}
The lower-triangular factor is multiplicative: when guidance is present, the component Cholesky factor is $L_s^\circ(x)M_{s,k}^\theta(x)$. The diagonal of $M_{s,k}^\theta$ is parameterized by a softplus transform initialized at one, while the off-diagonal entries are initialized at zero and scaled by a small factor. Thus the zero-initialized final layer gives $a_{s,k}^\theta=0$, $\Delta\mu_{s,k}^\theta=0$, and $M_{s,k}^\theta\approx I$, so the initial proposal is the guided Gaussian base. In the no-guidance ablation, the same network is initialized around the true local prior rather than around a guided proposal.

For continuous transitions, the affine proxy score is kept outside the network. Let $v_s=\tilde{H}_{s,T_s}^{-1}\tilde{\eta}_{s,T_s}$ be the proxy endpoint summary and let $\gamma_{\theta}$ denote Fourier embeddings. The residual control used in Equation~\eqref{eq:neural_control} is
\begin{equation}
    r_s^\theta(t,z)=R_\theta\!\left(z,\gamma_\theta(t/T_s),\gamma_\theta(T_s),v_s,\psi_s\right),
    \qquad
    g_s^\theta(t,z)=\nabla_z\log\tilde{h}_{s,t}(z)+r_s^\theta(t,z).
    \label{eq:app_continuous_residual_control}
\end{equation}
The score network first embeds $(z,\gamma_\theta(t/T_s))$ through a linear layer. Each hidden layer is then modulated by a FiLM map \citep{perez2018film} whose conditioning vector is $(v_s,\gamma_\theta(T_s),\psi_s)$: if $c_s(t)$ denotes this conditioning vector, the layer applies $h\mapsto (1+\Gamma_\theta(c_s(t)))\odot h+\Phi_\theta(c_s(t))$ before the nonlinearity, where $\Gamma_\theta, \Phi_\theta$ are independent scale and shift networks. The FiLM projections and final output layer are initialized at zero, so the residual control starts at zero and the initial controlled SDE is exactly the proxy-guided SDE. In the no-guidance ablation, both $\nabla_z\log\tilde{h}_{s,t}(z)$ is set to zero, so the network must learn the full control without backward-filtering messages. We implementation all the numerical experiments in JAX \footnote{https://github.com/jax-ml/jax} and neural networks with Flax \footnote{https://github.com/google/flax}. A full code implementation to reproduce all the experiential results presented in the paper is available at \footnote{https://github.com/gefanyang/nbffg}.

\subsection{Path-Subsampled Training}
\label{app:path_subsampling_details}

\begin{proposition}[Union-of-paths subsampling]
\label{prop:root_to_leaf_subsampling}
Let $\Gamma$ be a distribution over nonempty subsets of $\L$. Draw $B\sim\Gamma$ and write $A(B)=\{s\in\S:\exists l\in B\text{ such that }s\preceq l\}$ and $A_0(B)=A(B)\cup\{\0\}$, where $s\preceq l$ means that $s$ lies on the unique path from the root $\0$ to the leaf $l$. Define $\pi_s=\mathbb{P}_{B\sim\Gamma}\{s\in A(B)\}$ and $\rho_l=\mathbb{P}_{B\sim\Gamma}\{\pa{l}\in A_0(B)\}$, and assume $\pi_s>0$ for all $s\in\S$ and $\rho_l>0$ for all $l\in\L$. If only the variational transitions on $A(B)$ are simulated, then the Rao--Blackwellized path loss
\[
    \widehat{\J}_{B}(\theta)
    =
    \sum_{s\in A(B)}
    \frac{C_s^\theta}{\pi_s}
    +
    \sum_{\substack{l\in\L\\\pa{l}\in A_0(B)}}
    \frac{C_l}{\rho_l}
\]
satisfies
\[
    \mathbb{E}_{B\sim\Gamma}\mathbb{E}_{\Q^\theta_{A(B)}}\!\left[\widehat{\J}_{B}(\theta)\right]=\J(\theta),
\]
where $\Q^\theta_{A(B)}$ is the marginal variational law on the sampled path union.
\end{proposition}

\begin{proof}
For any internal vertex $s\in\S$, the factorization of $\Q^\theta$ implies that the marginal law of the variables on the path from $\0$ to $s$ is the same whether we sample the full tree or only a sampled path union containing $s$. Hence, whenever $s\in A(B)$,
\[
    \mathbb{E}_{\Q^\theta_{A(B)}}[C_s^\theta]
    =
    \mathbb{E}_{\Q^\theta}[C_s^\theta].
\]
Therefore the expected internal-edge part of the path loss is
\begin{align*}
    &\mathbb{E}_{B\sim\Gamma}
    \mathbb{E}_{\Q^\theta_{A(B)}}
    \left[
        \sum_{s\in A(B)}
        \frac{C_s^\theta}{\pi_s}
    \right]
    \\
    &\quad =
    \sum_{s\in\S}
    \frac{1}{\pi_s}
    \mathbb{P}_{B\sim\Gamma}\{s\in A(B)\}
    \mathbb{E}_{\Q^\theta}[C_s^\theta]
    =
    \sum_{s\in\S}
    \mathbb{E}_{\Q^\theta}[C_s^\theta].
 \end{align*}
For a leaf $l\in\L$, the cost $C_l$ depends only on $X_{\pa{l}}$ and the fixed observed value $x_l$. Whenever $\pa{l}\in A_0(B)$, the marginal law of $X_{\pa{l}}$ under $\Q^\theta_{A(B)}$ agrees with its full-tree marginal; if $\pa{l}=\0$, this state is fixed. Thus the Rao--Blackwellized leaf term satisfies
\begin{align*}
    &\mathbb{E}_{B\sim\Gamma}
    \mathbb{E}_{\Q^\theta_{A(B)}}
    \left[
        \sum_{\substack{l\in\L\\\pa{l}\in A_0(B)}}
        \frac{C_l}{\rho_l}
    \right]
    \\
    &\quad =
    \sum_{l\in\L}
    \frac{1}{\rho_l}
    \mathbb{P}_{B\sim\Gamma}\{\pa{l}\in A_0(B)\}
    \mathbb{E}_{\Q^\theta}[C_l]
    =
    \sum_{l\in\L}
    \mathbb{E}_{\Q^\theta}[C_l].
\end{align*}
Adding the two identities gives the stated unbiasedness.
\end{proof}

Here $C_s^\theta$ is the internal-edge ELBO cost: the log variational-to-prior density ratio in discrete time and the control-energy contribution in continuous time; $C_l=-\log p_l(x_l\mid X_{\pa{l}})$. The path-subsampled objective in Proposition~\ref{prop:root_to_leaf_subsampling} uses a Horvitz--Thompson correction over the sampled path union. A plain implementation samples $B\subseteq\L$, simulates the internal vertices in $A(B)$, and weights both sampled internal costs and sampled leaf costs by their inclusion probabilities. The version stated in Equation~\eqref{eq:path_subsampled_elbo} additionally uses an optional Rao--Blackwellized leaf term: once $X_{\pa{l}}$ has been sampled, the likelihood costs of all observed children of that parent are available and can be averaged into the estimator, rather than using only the selected leaves.

For uniform sampling without replacement with $|B|=m$ and $N=|\L|$, if $\L_s$ denotes the set of leaves below an internal vertex $s$, then $\pi_s=1-\binom{N-|\L_s|}{m}/\binom{N}{m}$. For a leaf $l$, the Rao--Blackwellized parent-inclusion probability is $\rho_l=\pi_{\pa{l}}$ when $\pa{l}\in\S$ and $\rho_l=1$ when $\pa{l}=\0$. The single-path estimator is recovered at $m=1$; larger leaf mini-batches reduce variance because shared ancestors are simulated once and reused by all sampled descendants.

\subsection{Amortized Correction and Path Encoding}
\label{app:amortized_correction_details}

For each vertex $s$, let $u_{s,j}$ be the structural feature vector at the $j$th step of the padded root-to-$s$ path, and let $a_{s,j}\in\{0,1\}$ be its mask. In our implementation, $u_{s,j}$ contains normalized sibling rank, normalized branching factor, subtree leaf fraction, normalized depth, and optionally log edge length in the continuous formulation. The path embedding is computed by a masked GRU,
\begin{equation}
    \bar{h}_{s,j}=\operatorname{GRU}_{\theta}(\ell_{\theta}(u_{s,j}),h_{s,j-1}), \qquad h_{s,j}=a_{s,j}\bar{h}_{s,j}+(1-a_{s,j})h_{s,j-1}, \qquad \psi_s=h_{s,D},
    \label{eq:app_path_encoder}
\end{equation}
where $\ell_{\theta}$ is a learned linear projection and $D$ is the maximum padded depth. The shared correction networks use $\psi_s$ as the structural edge context, so every sampled path updates parameters reused by unsampled branches.

\section{Additional Experimental Details}
\label{app:experimental_details}

\subsection{Benchmark 1A: Discrete Gaussian Tree}
\label{app:exp1a_details}

The true hidden transition on each non-leaf edge is \(X_v=A_vX_{\pa{v}}+b_v+\epsilon_v\), where \(\epsilon_v\sim\mathcal{N}(0,Q)\), and each observed leaf is generated from \(X_l\mid X_{\pa{l}}\sim\mathcal{N}(X_{\pa{l}},r_{\mathrm{obs}}^2I)\) with \(r_{\mathrm{obs}}=0.05\). The root is fixed at \(x_0=0\). All runs use the same tree topology and transition parameters; each run changes only the single forward sample used to generate the observed leaves and the downstream training randomness.

The latent tree is an irregular binary backbone generated by repeated splits. Starting from the root, we split uniformly chosen active terminal latent vertices until \(K=7\) splits have been made, rejecting and resampling if the latent-backbone depth before attaching observations exceeds \(D_{\max}=5\). This gives \(N_{\mathrm{lat}}=1+2K=15\) latent vertices and \(8\) terminal latent vertices. We then attach one observed leaf to each terminal latent vertex, giving an augmented tree with \(23\) total vertices. Figure~\ref{fig:exp1_trees} shows this augmented tree, so the displayed topology includes the additional observation leaves and can be one level deeper than the latent backbone controlled by \(D_{\max}\).

For \(d=4\), let \(U_d\) be a fixed random orthogonal matrix and define \(A_0=U_d\operatorname{diag}(\lambda_1,\ldots,\lambda_d)U_d^\top\), where \(\lambda_i=0.35+0.50(i-1)/(d-1)\). Each hidden edge uses \(A_v=\rho_vA_0\) with \(\rho_v\sim\operatorname{Unif}(0.85,1.05)\), and \(b_v=0.15g_v/\sqrt{d}\) with \(g_v\sim\mathcal{N}(0,I_d)\). The shared process covariance is \(Q=U_d\operatorname{diag}(q_1^2,\ldots,q_d^2)U_d^\top\), where \(q_i=0.05+0.07(i-1)/(d-1)\).

All guided proxy models keep the true covariance and true leaf observation model. We evaluate \((\tilde{A}_v,\tilde{b}_v)\in\{(A_v,b_v),(I,0),(-A_v,b_v),(A_v,-b_v),(-A_v,-b_v)\}\). For each proxy, ``uncorrected'' denotes the guided proposal built from the proxy messages, and ``corrected'' denotes NBFFG using the same proxy guide after neural correction. We additionally evaluate a no-guidance ablation, where the neural correction is trained from the prior without proxy messages.

The exact log evidence and exact posterior marginals are computed by Gaussian tree smoothing, or more specifically, Rauch-Tung-Striebel (RTS) smoother. Let \(\mathcal{J}^\star=-\log p(x_\L)\), let \(p_s^\star=\mathcal{N}(m_s^\star,C_s^\star)\) be the exact posterior marginal at internal latent vertex \(s\in\S\), and let \(q_s=\mathcal{N}(\hat{m}_s,\hat{C}_s)\) be the empirical Gaussian marginal fitted from posterior samples. We report
\begin{equation}
\begin{aligned}
\Delta_{\mathrm{rel}} &= \frac{\mathcal{J}-\mathcal{J}^\star}{|\mathcal{J}^\star|}, &
\mathrm{KL}_{\mathrm{avg}} &= \frac{1}{|\S|}\sum_{s\in\S}\mathrm{KL}(q_s\Vert p_s^\star), \\
E_{\mathrm{mean}} &= \frac{1}{|\S|}\sum_{s\in\S}\|\hat{m}_s-m_s^\star\|_2, &
E_{\mathrm{cov}} &= \frac{1}{|\S|}\sum_{s\in\S}\frac{\|\hat{C}_s-C_s^\star\|_F}{\|C_s^\star\|_F},
\end{aligned}
\label{eq:exp1_metrics}
\end{equation}
where $\|\cdot\|_F$ is the Frobenius norm. Gains are computed as \(1-M_{\mathrm{corrected}}/M_{\mathrm{uncorrected}}\) for the considered metric $M$.

The corrected model uses a 3-layer neural correction network equipped with SiLU activations, and with embedding dimension \(8\) for $\psi_s$, trained for \(10{,}000\) iterations with \(32\) particles per iteration, peak learning rate \(10^{-3}\), \(500\) warmup steps, cosine decay to \(10\%\) of the peak learning rate, and global gradient clipping at \(1.0\). Evaluation uses \(16\) batches of \(128\) particles for NELBO estimation and \(128\) posterior samples for marginal diagnostics.

For the optimal proxy in the discrete benchmark, the uncorrected guided proposal coincides with the exact Gaussian smoothing law. The relative NELBO gap of the uncorrected guide is therefore reported as zero. The remaining nonzero optimal-proxy diagnostics in Table~\ref{tab:exp1a_proxy_correction}, including the corrected entries, come from Monte Carlo estimation and training/evaluation randomness under the same experimental pipeline. We report them as reference values rather than as evidence of proxy misspecification.

\begin{table}[t]
\centering
\caption{Benchmark 1A: discrete heterogeneous linear Gaussian tree under proxy misspecification. Values are mean \(\pm\) std. over five runs. The proxy column gives \((\tilde A,\tilde b)\), with ``No guidance'' denoting an unguided prior-based ablation. ``Uncorr.'' denotes the proxy guide before neural correction, ``Corr.'' denotes NBFFG after neural correction, and gain is \(1-\mathrm{Corr.}/\mathrm{Uncorr.}\). For the optimal proxy, corrected diagnostics are reported as Monte Carlo reference errors, while gain is not interpreted.}
\label{tab:exp1a_proxy_correction}
\setlength{\tabcolsep}{2.4pt}
\scriptsize
\resizebox{\textwidth}{!}{%
\begin{tabular}{@{}lrrrrrrrrrrrr@{}}
\toprule
& \multicolumn{3}{c}{Relative NELBO gap \(\downarrow\)}
& \multicolumn{3}{c}{Average marginal KL \(\downarrow\)}
& \multicolumn{3}{c}{Marginal mean error \(\downarrow\)}
& \multicolumn{3}{c}{Marginal covariance error \(\downarrow\)} \\
\cmidrule(lr){2-4}\cmidrule(lr){5-7}\cmidrule(lr){8-10}\cmidrule(l){11-13}
Proxy
& Uncorr. & Corr. & Gain \(\uparrow\)
& Uncorr. & Corr. & Gain \(\uparrow\)
& Uncorr. & Corr. & Gain \(\uparrow\)
& Uncorr. & Corr. & Gain \(\uparrow\) \\
\midrule
No guidance
& \(1.70{\times}10^4\pm7.34{\times}10^3\) & \(96.659\pm15.237\) & \(99.4\%\)
& \(9.670\pm2.548\) & \(75.495\pm24.010\) & \(-680.7\%\)
& \(0.133\pm0.041\) & \(0.635\pm0.138\) & \(-378.1\%\)
& \(5.475\pm0.165\) & \(0.820\pm0.077\) & \(85.0\%\) \\
\((A,b)\)
& \(0.000\pm0.000\) & \(0.002\pm0.001\) & n/a
& \(0.056\pm0.004\) & \(0.060\pm0.010\) & n/a
& \(0.009\pm0.001\) & \(0.009\pm0.001\) & n/a
& \(0.185\pm0.013\) & \(0.198\pm0.011\) & n/a \\
\((I,0)\)
& \(0.913\pm0.230\) & \(0.035\pm0.011\) & \(96.2\%\)
& \(1.861\pm0.096\) & \(0.118\pm0.004\) & \(93.7\%\)
& \(0.092\pm0.003\) & \(0.015\pm0.001\) & \(83.2\%\)
& \(0.273\pm0.010\) & \(0.234\pm0.007\) & \(14.1\%\) \\
\((-A,b)\)
& \(1.227\pm0.322\) & \(0.016\pm0.005\) & \(98.7\%\)
& \(1.677\pm0.602\) & \(0.080\pm0.017\) & \(95.2\%\)
& \(0.104\pm0.019\) & \(0.013\pm0.002\) & \(87.3\%\)
& \(0.180\pm0.022\) & \(0.197\pm0.013\) & \(-9.3\%\) \\
\((A,-b)\)
& \(3.352\pm0.800\) & \(0.046\pm0.015\) & \(98.6\%\)
& \(6.559\pm0.051\) & \(0.096\pm0.022\) & \(98.5\%\)
& \(0.181\pm0.001\) & \(0.016\pm0.002\) & \(91.2\%\)
& \(0.185\pm0.013\) & \(0.203\pm0.012\) & \(-9.9\%\) \\
\((-A,-b)\)
& \(4.482\pm0.743\) & \(0.078\pm0.040\) & \(98.2\%\)
& \(6.690\pm1.073\) & \(0.111\pm0.020\) & \(98.3\%\)
& \(0.192\pm0.025\) & \(0.016\pm0.002\) & \(91.5\%\)
& \(0.180\pm0.022\) & \(0.216\pm0.016\) & \(-20.1\%\) \\
\bottomrule
\end{tabular}%
}
\end{table}

\subsection{Benchmark 1B: Continuous OU Tree}
\label{app:exp1b_details}

The true hidden dynamics on each non-leaf edge are \(dZ_v(t)=B_v(\theta_v-Z_v(t))dt+\sigma dW_t\), with \(Z_v(0)=X_{\pa{v}}\) and \(X_v=Z_v(T_v)\). Each observed leaf is generated from \(X_l\mid X_{\pa{l}}\sim\mathcal{N}(X_{\pa{l}},r_{\mathrm{obs}}^2I)\) with \(r_{\mathrm{obs}}=0.05\), and the root is fixed at \(x_0=0\). As in Benchmark~1A, all runs share one topology and one prior model, and each run produces one observation instance.

The latent tree is generated by the same random-split procedure as Benchmark~1A, using \(K=7\) splits and latent-backbone depth limit \(D_{\max}=5\) before observations are attached. This gives \(15\) latent vertices and \(8\) observed leaves after augmentation. The state dimension is \(d=2\). Edge lengths are sampled as \(T_v=0.4+0.6u_v\), with \(u_v\sim\operatorname{Unif}(0,1)\). Figure~\ref{fig:exp1_trees} again shows the augmented tree including the observed leaves.

\begin{figure}[ht]
    \centering
    \begin{minipage}{0.47\textwidth}
        \centering
        \includegraphics[width=\linewidth]{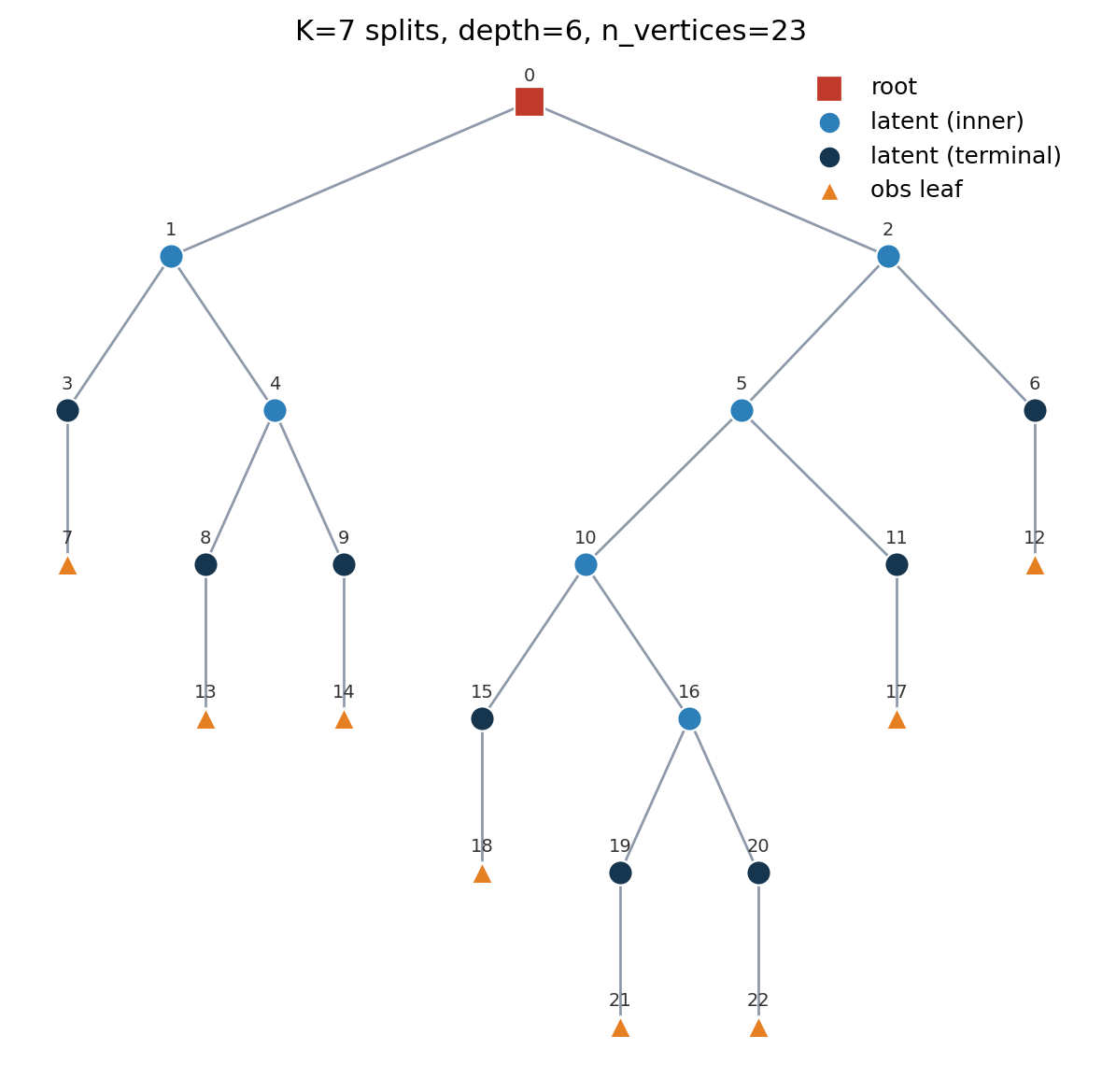}
        \textbf{(a) Benchmark 1A}
    \end{minipage}
    \hfill
    \begin{minipage}{0.47\textwidth}
        \centering
        \includegraphics[width=\linewidth]{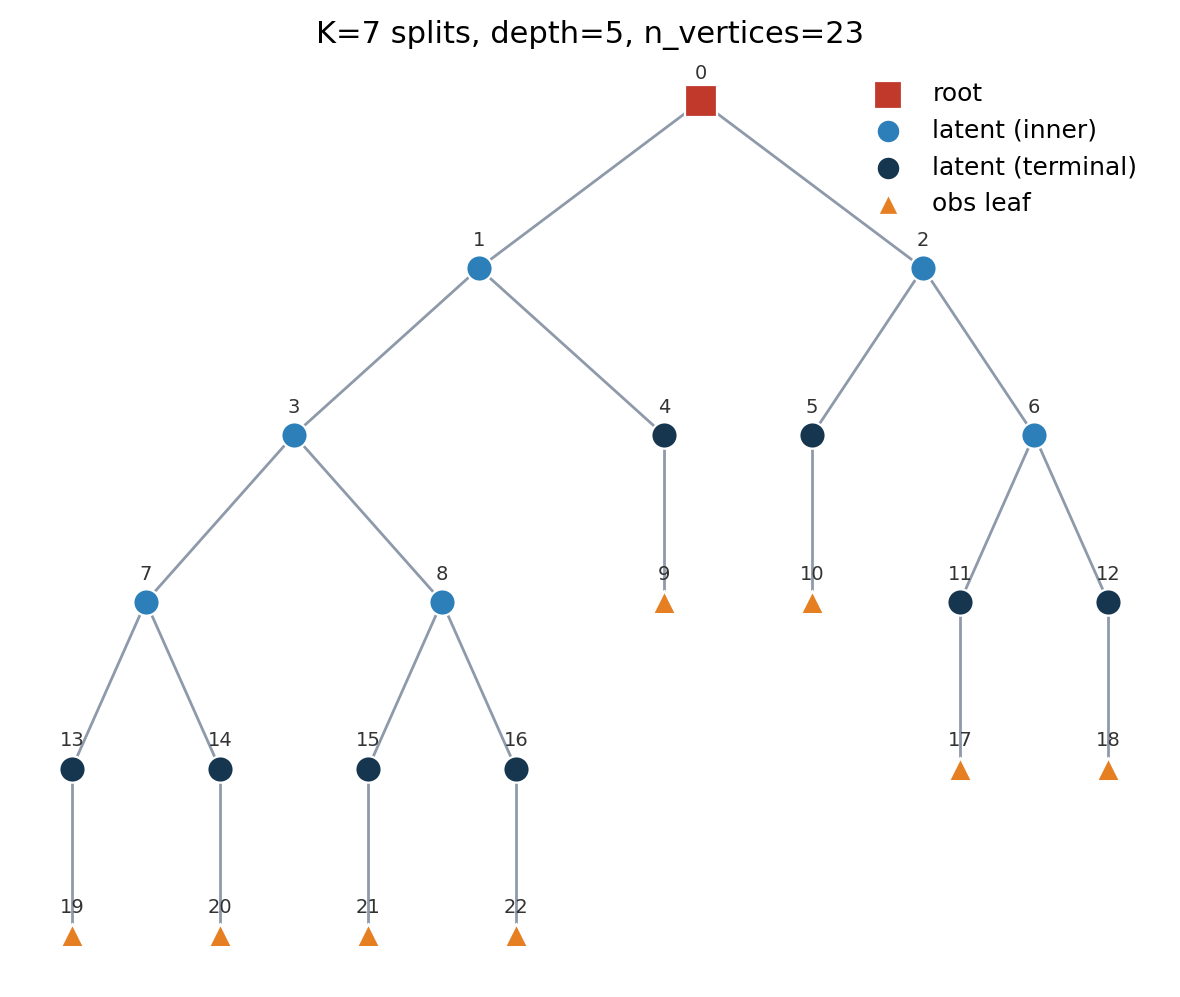}
        \textbf{(b) Benchmark 1B}
    \end{minipage}
    \caption{Tree topologies used in the exact-posterior benchmarks. Blue vertices are latent states and orange leaves are observed children attached to terminal latent vertices.}
    \label{fig:exp1_trees}
\end{figure}

Let \(U=2^{-1/2}\begin{psmallmatrix}1&-1\\1&1\end{psmallmatrix}\) and \(B_0=U\operatorname{diag}(0.6,1.2)U^\top\). Each hidden edge uses \(B_v=\rho_vB_0\), with \(\rho_v\sim\operatorname{Unif}(0.85,1.15)\), and \(\theta_v=0.50g_v/\sqrt{d}\), with \(g_v\sim\mathcal{N}(0,I_d)\). The shared diffusion covariance is \(a=\sigma\sigma^\top=U\operatorname{diag}(0.15^2,0.30^2)U^\top\), and the same \(\sigma\) is used by all proxy models and by the variational SDE.

All guided proxy models keep the true diffusion and true leaf observation model. We evaluate the optimal proxy \((\tilde{B}_v,\tilde{\theta}_v)=(B_v,\theta_v)\), the canonical Brownain proxy \(d\tilde{Z}_v(t)=\sigma d\tilde{W}_t\), and \((\tilde{B}_v,\tilde{\theta}_v)\in\{(DB_vD,\theta_v),(B_v,-\theta_v),(DB_vD,-\theta_v)\}\), where \(D=\operatorname{diag}(1,-1)\). We use \(DB_vD\) rather than \(-B_v\) so that the proxy remains mean-reverting while reversing the cross-coordinate coupling. We also include a no-guidance ablation, where the neural correction is trained from the prior SDE without any guiding message.

For the exact reference, each OU edge is converted to its closed-form endpoint transition \(X_v\mid X_{\pa{v}}\sim\mathcal{N}(M_vX_{\pa{v}}+c_v,Q_v^{\mathrm{OU}})\), where \(M_v=\exp(-B_vT_v)\), \(c_v=(I-M_v)\theta_v\), and \(Q_v^{\mathrm{OU}}=\int_0^{T_v}\exp(-B_v\tau)a\exp(-B_v^\top \tau)d\tau\). Gaussian tree smoothing on these endpoint transitions gives the exact log evidence and posterior marginals. We use the metrics in Equation~\eqref{eq:exp1_metrics} and the same gain definition as in Appendix~\ref{app:exp1a_details}.

The variational SDE is simulated with \(50\) Euler--Maruyama steps per edge. The corrected model uses a 3-layer score network equipped with SiLU activations, and with embedding dimension \(8\) of $\psi_s$, trained for \(10{,}000\) iterations with \(16\) particles per iteration, peak learning rate \(10^{-3}\), \(500\) warmup steps, cosine decay to \(5\%\) of the peak learning rate, and global gradient clipping at \(1.0\). Evaluation uses \(16\) batches of \(128\) particles for NELBO estimation and \(128\) endpoint samples for marginal diagnostics.

Figure~\ref{fig:exp1_loss_curves} reports the training loss curves. The legends use implementation identifiers rather than the proxy notation used in the tables: in Benchmark~1A, \texttt{canonical} denotes \((I,0)\), \texttt{sign\_flip\_A} denotes \((-A,b)\), \texttt{sign\_flip\_b} denotes \((A,-b)\), and \texttt{sign\_flip\_Ab} denotes \((-A,-b)\); in Benchmark~1B, \texttt{canonical\_brownian} denotes the Brownian proxy, \texttt{sign\_flip\_coupling} denotes \((DBD,\theta)\), \texttt{sign\_flip\_target} denotes \((B,-\theta)\), and \texttt{sign\_flip\_coupling\_target} denotes \((DBD,-\theta)\). In both panels, \texttt{optimal} denotes the matching-transition proxy and \texttt{no\_guidance} denotes the prior-based ablation without proxy messages.

\begin{figure}[ht]
    \centering
    \begin{minipage}{0.48\textwidth}
        \centering
        \includegraphics[width=\linewidth]{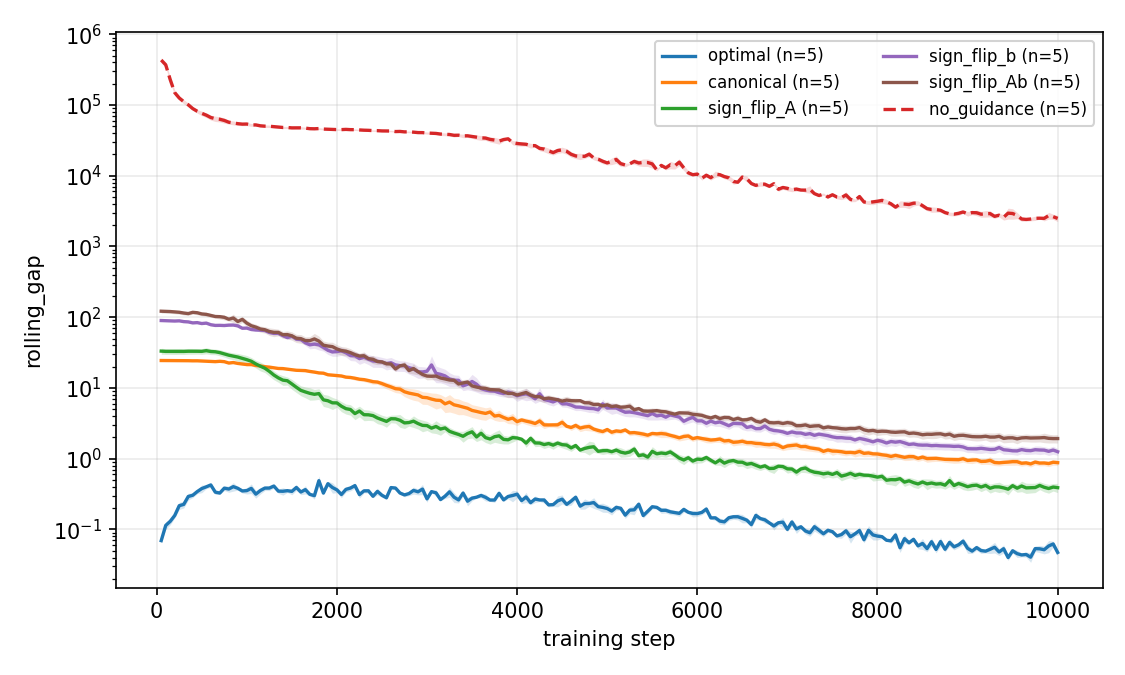}
        \textbf{(a) Benchmark 1A}
    \end{minipage}
    \hfill
    \begin{minipage}{0.48\textwidth}
        \centering
        \includegraphics[width=\linewidth]{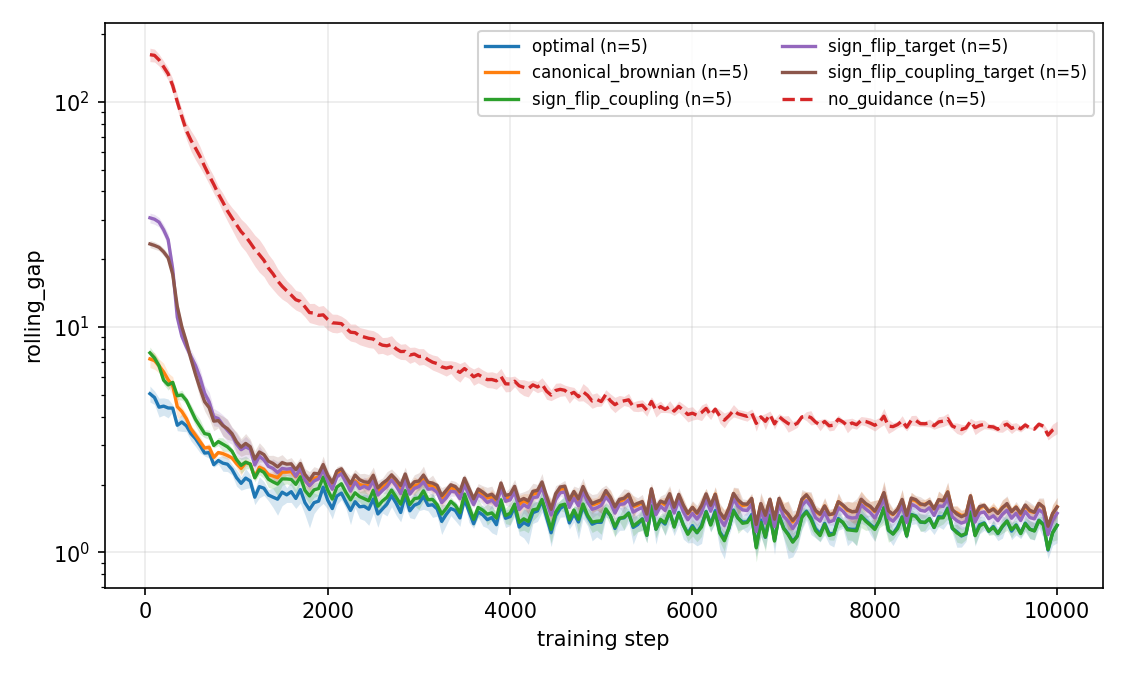}
        \textbf{(b) Benchmark 1B}
    \end{minipage}
    \caption{Training loss curves for the exact-posterior benchmarks. Guided proxies converge rapidly to low NELBO values, while the no-guidance ablations remain substantially worse, illustrating the role of BFFG guidance as a useful base distribution.}
    \label{fig:exp1_loss_curves}
\end{figure}

For the optimal proxy in the continuous benchmark, corrected results are reported for all metrics because the variational law is simulated with a fixed Euler--Maruyama discretization. Even when the proxy drift and diffusion match the true OU dynamics, the reported diagnostics include both Monte Carlo error and SDE time-discretization error. These entries should be interpreted as a reference for the best performance attainable under the current discretization and evaluation strategy, not as evidence of posterior misspecification by the method.

\begin{table}[t]
\centering
\caption{Benchmark 1B: continuous OU tree under proxy misspecification. Values are mean \(\pm\) std. over five runs. The proxy column gives choices of \((\tilde B,\tilde\theta)\), with ``No guidance'' denoting an unguided prior-based ablation. ``Uncorr.'' denotes the proxy guide before neural correction, ``Corr.'' denotes NBFFG after neural correction, and gain is \(1-\mathrm{Corr.}/\mathrm{Uncorr.}\). For the optimal proxy, corrected results are also meaningful because the controlled SDE objective is evaluated under a fixed time discretization.}
\label{tab:exp1b_proxy_correction}
\setlength{\tabcolsep}{2.4pt}
\scriptsize
\resizebox{\textwidth}{!}{%
\begin{tabular}{@{}lrrrrrrrrrrrr@{}}
\toprule
& \multicolumn{3}{c}{Relative NELBO gap \(\downarrow\)}
& \multicolumn{3}{c}{Average marginal KL \(\downarrow\)}
& \multicolumn{3}{c}{Marginal mean error \(\downarrow\)}
& \multicolumn{3}{c}{Marginal covariance error \(\downarrow\)} \\
\cmidrule(lr){2-4}\cmidrule(lr){5-7}\cmidrule(lr){8-10}\cmidrule(l){11-13}
Proxy
& Uncorr. & Corr. & Gain \(\uparrow\)
& Uncorr. & Corr. & Gain \(\uparrow\)
& Uncorr. & Corr. & Gain \(\uparrow\)
& Uncorr. & Corr. & Gain \(\uparrow\) \\
\midrule
No guidance
& \(29.245\pm17.626\) & \(0.605\pm0.226\) & \(97.9\%\)
& \(11.904\pm1.922\) & \(0.144\pm0.022\) & \(98.8\%\)
& \(0.173\pm0.034\) & \(0.011\pm0.001\) & \(93.8\%\)
& \(7.159\pm0.330\) & \(0.514\pm0.058\) & \(92.8\%\) \\
\((B,\theta)\)
& \(0.892\pm0.583\) & \(0.213\pm0.075\) & \(76.2\%\)
& \(0.261\pm0.051\) & \(0.048\pm0.011\) & \(81.5\%\)
& \(0.032\pm0.004\) & \(0.010\pm0.001\) & \(69.6\%\)
& \(0.587\pm0.040\) & \(0.264\pm0.023\) & \(54.9\%\) \\
\((0, 0)\)
& \(1.273\pm0.724\) & \(0.262\pm0.104\) & \(79.4\%\)
& \(0.387\pm0.121\) & \(0.062\pm0.008\) & \(84.0\%\)
& \(0.054\pm0.012\) & \(0.009\pm0.001\) & \(82.4\%\)
& \(0.404\pm0.021\) & \(0.313\pm0.033\) & \(22.7\%\) \\
\((DBD,\theta)\)
& \(1.294\pm0.648\) & \(0.214\pm0.073\) & \(83.5\%\)
& \(0.431\pm0.051\) & \(0.045\pm0.004\) & \(89.5\%\)
& \(0.041\pm0.008\) & \(0.009\pm0.001\) & \(77.0\%\)
& \(0.603\pm0.041\) & \(0.257\pm0.017\) & \(57.4\%\) \\
\((B,-\theta)\)
& \(5.122\pm1.722\) & \(0.244\pm0.092\) & \(95.2\%\)
& \(1.381\pm0.208\) & \(0.056\pm0.008\) & \(96.0\%\)
& \(0.093\pm0.005\) & \(0.010\pm0.002\) & \(89.7\%\)
& \(0.587\pm0.040\) & \(0.291\pm0.028\) & \(50.4\%\) \\
\((DBD,-\theta)\)
& \(3.955\pm1.470\) & \(0.264\pm0.111\) & \(93.3\%\)
& \(1.105\pm0.105\) & \(0.062\pm0.009\) & \(94.4\%\)
& \(0.083\pm0.005\) & \(0.010\pm0.002\) & \(88.1\%\)
& \(0.603\pm0.041\) & \(0.317\pm0.026\) & \(47.3\%\) \\
\bottomrule
\end{tabular}%
}
\end{table}

\subsection{Benchmark 1C: Linear Gaussian Path Subsampling}
\label{app:exp1c_details}

Benchmark~1C uses a wider discrete linear-Gaussian tree to test the accuracy--speed trade-off of path-subsampled training. The hidden transition is the same family as Benchmark~1A,
\begin{equation}
    X_v=A_vX_{\pa{v}}+b_v+\epsilon_v,\qquad \epsilon_v\sim\mathcal{N}(0,Q),
    \label{eq:exp1c_hidden_transition}
\end{equation}
with \(d=4\) and a fixed root \(x_0=0\). The canonical random-walk proxy \((\tilde{A}_v,\tilde{b}_v)=(I,0)\) is used in all runs, with the true covariance \(Q\). Unlike Benchmark~1A, each observed leaf uses an edge-specific orthogonal observation operator,
\begin{equation}
    X_l\mid X_{\pa{l}}\sim\mathcal{N}(G_lX_{\pa{l}},r_{\mathrm{obs}}^2I),\qquad G_l\in O(d),\qquad r_{\mathrm{obs}}=0.5.
    \label{eq:exp1c_observation_model}
\end{equation}
The matrices \(G_l\) are fixed across all runs and generated once from the model seed. Specifically, for each observed leaf \(l\), we draw \(Z_l\in\mathbb{R}^{d\times d}\) with i.i.d. standard normal entries, compute a QR decomposition \(Z_l=Q_lR_l\), draw an independent diagonal sign matrix \(D_l=\operatorname{diag}(\xi_{l,1},\ldots,\xi_{l,d})\) with \(\xi_{l,i}\in\{-1,1\}\), and set \(G_l=Q_lD_l\). This gives each observed sibling a distinct rotated view of the same terminal latent state while preserving signal scale and numerical conditioning. The exact reference is again obtained by Gaussian tree smoothing.

The latent topology is a balanced \(4\)-ary tree with three latent layers: one root, four intermediate latent vertices, and sixteen terminal latent vertices. Four observed leaves are attached to each terminal latent vertex, giving \(21\) latent vertices, \(64\) observed leaves, and \(85\) total vertices in the augmented tree. Figure~\ref{fig:exp1_subsampling_trees} (a) shows this topology and highlights the multiple observation leaves attached to each terminal latent state.

We compare three training strategies. Full-tree training uses all \(64\) observations per iteration and serves as the time and accuracy reference. Direct leaf-path subsampling samples \(m\in\{1,2,4,8,16,32\}\) observed leaves per iteration and reweights their contributions. Rao--Blackwellized terminal-path subsampling samples \(m\in\{1,2,4,8\}\) terminal latent vertices per iteration and includes all four observed leaves attached to each sampled terminal, corresponding to \(4m\) effective observations per step. Evaluation is always performed on the full tree against the exact smoothing posterior.

All models use the same 5-layer correction network with hidden dimension \(32\), embedding dimension \(16\), \(32\) training particles per iteration, \(30{,}000\) iterations, peak learning rate \(10^{-3}\), \(5000\) warmup steps, cosine decay to \(10\%\) of the peak learning rate, and gradient clipping at \(5.0\). Evaluation uses \(16\) batches of \(128\) particles for NELBO estimation and \(128\) posterior samples for marginal diagnostics.

Figure~\ref{fig:exp1_subsampling_scaling} (a) plots the resulting accuracy--speed trade-off. Full-tree training achieves \(\mathrm{KL}_{\mathrm{avg}}=0.066\pm0.008\). The most aggressive settings with one or two effective observations per step are substantially faster but too noisy to be reliable. Around eight effective observations per step, the methods retain a useful speed advantage, about \(1.7\times\) relative to full-tree training, while reaching KL error around \(0.09\). Increasing the effective observations further reduces KL but also erodes the speedup. In this benchmark, Rao--Blackwellized terminal sampling closely tracks direct leaf-path sampling at matched effective observations per step. This is consistent with the conditional independence of the attached observations; stronger gains from Rao--Blackwellization would likely require correlated or otherwise structured observation groups.

\subsection{Benchmark 1D: OU Path Subsampling}
\label{app:exp1d_details}

Benchmark~1D repeats the subsampling study for continuous OU edges. The hidden dynamics are the same family as Benchmark~1B,
\begin{equation}
    dZ_v(t)=B_v(\theta_v-Z_v(t))dt+\sigma dW_t,\qquad Z_v(0)=X_{\pa{v}},\qquad X_v=Z_v(T_v),
    \label{eq:exp1d_ou_transition}
\end{equation}
with \(d=2\), fixed root \(x_0=0\), and the canonical Brownian proxy \(d\tilde{Z}_v(t)=\sigma d\tilde{W}_t\). Each observed leaf uses \(X_l\mid X_{\pa{l}}\sim\mathcal{N}(X_{\pa{l}},r_{\mathrm{obs}}^2I)\) with \(r_{\mathrm{obs}}=0.1\). Exact references are computed from the closed-form OU endpoint transitions as in Benchmark~1B.

The latent tree is intentionally small and shallow due to the growth of continuous integration steps: the root has four terminal latent children, and each terminal latent vertex has four observed leaves. Thus the augmented tree contains \(5\) latent vertices, \(16\) observed leaves, and \(21\) total vertices. Figure~\ref{fig:exp1_subsampling_trees} (b) shows this topology.

\begin{figure}[ht]
    \centering
    \begin{minipage}{1.0\textwidth}
        \centering
        \includegraphics[width=\linewidth]{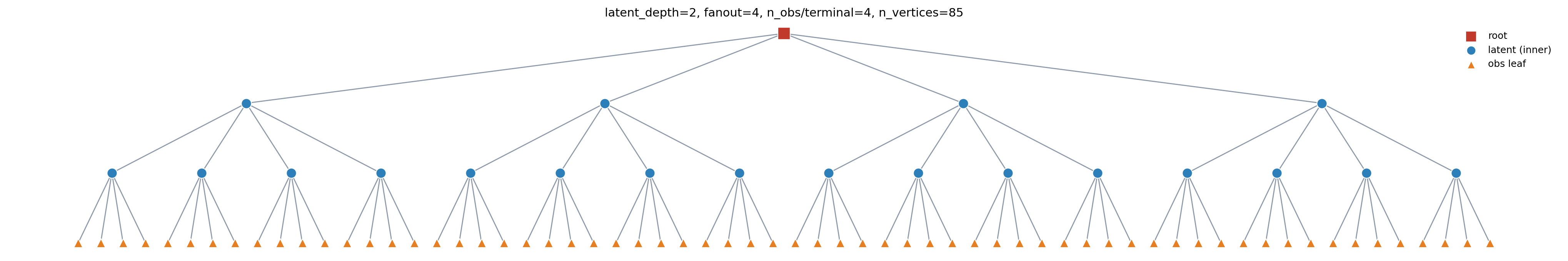}
        \textbf{(a) Benchmark 1C}
    \end{minipage}
    \vspace{6pt}
    \begin{minipage}{0.8\textwidth}
        \centering
        \includegraphics[width=\linewidth]{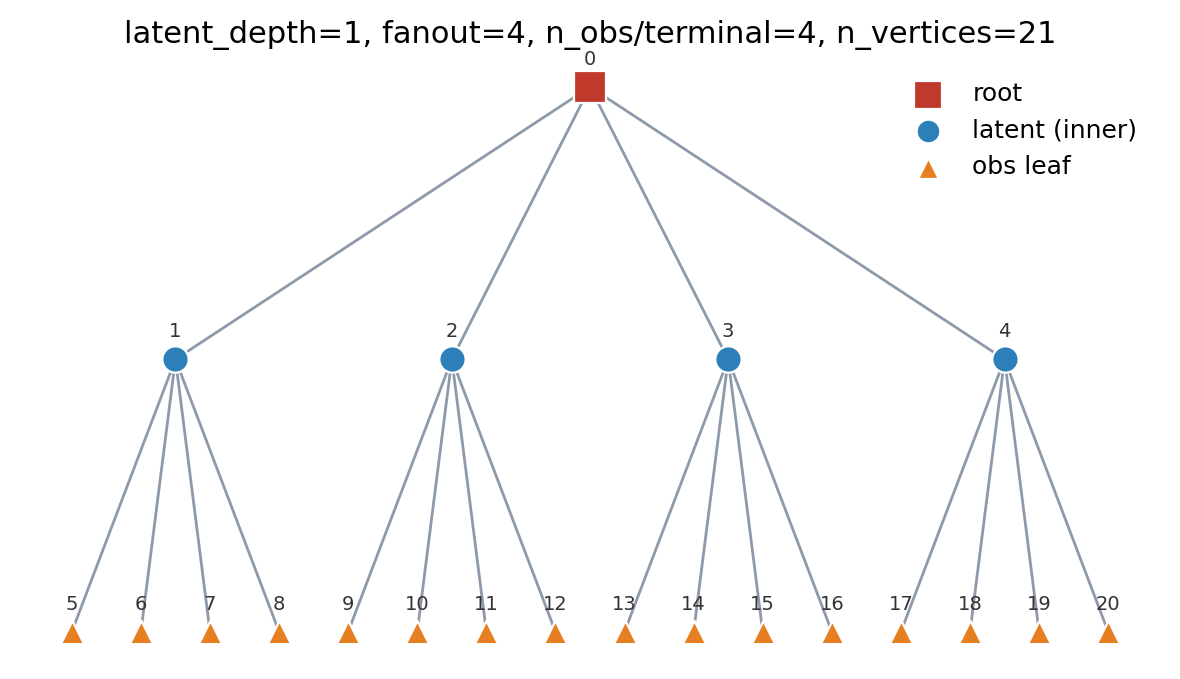}
        \textbf{(b) Benchmark 1D}
    \end{minipage}
    \caption{Tree topologies for the path-subsampling benchmarks. Benchmark~1C uses a wider discrete linear-Gaussian tree, while Benchmark~1D uses a smaller shallow OU tree. In both cases, multiple observed leaves are attached to each terminal latent vertex.}
    \label{fig:exp1_subsampling_trees}
\end{figure}

The training strategies mirror Benchmark~1C. Full-tree training uses all \(16\) observations per iteration. Direct leaf-path subsampling uses \(m\in\{1,2,4,8,16\}\) observed leaves per iteration. Rao--Blackwellized terminal-path subsampling uses \(m\in\{1,2,4\}\) terminal latent vertices per iteration, corresponding to \(4m\) effective observations. The case \(m=4\) covers all terminal groups and is therefore a sanity check against full-tree training, up to subsampler overhead.

All variational SDEs use \(50\) Euler--Maruyama steps per hidden edge. The correction network has five layers, hidden dimension \(32\), node embedding dimension \(16\), time embedding dimension \(16\), \(16\) training particles per iteration, \(10{,}000\) iterations, peak learning rate \(10^{-3}\), \(1000\) warmup steps, cosine decay to \(10\%\) of the peak learning rate, and no gradient clipping. Evaluation uses \(16\) batches of \(128\) particles for NELBO estimation and \(128\) endpoint samples for marginal diagnostics.

Figure~\ref{fig:exp1_subsampling_scaling} (b) shows a qualitatively different regime from Benchmark~1C. Full-tree training gives \(\mathrm{KL}_{\mathrm{avg}}=0.061\pm0.016\). Even the most aggressive direct leaf-path sampler, using one observation per step, remains stable with \(\mathrm{KL}_{\mathrm{avg}}=0.110\pm0.025\) and \(2.57\times\) speedup. This is likely because the tree is much smaller and shallower than in Benchmark~1C, and because the uncorrected Brownian guide is already close to the exact posterior on this short OU tree, with uncorrected \(\mathrm{KL}_{\mathrm{avg}}=0.255\pm0.078\). The continuous setting also yields larger speedups because skipping an edge avoids simulating many Euler--Maruyama steps.

Rao--Blackwellized terminal sampling improves over direct leaf-path sampling in this benchmark. At matched effective observations, the differences are small but consistent; for example, at eight effective observations per step, direct leaf sampling gives \(\mathrm{KL}_{\mathrm{avg}}=0.073\pm0.017\) and \(1.39\times\) speedup, while Rao--Blackwellized sampling gives \(\mathrm{KL}_{\mathrm{avg}}=0.069\pm0.016\) and \(1.38\times\) speedup. The full-coverage Rao--Blackwellized case matches full-tree training, with \(\mathrm{KL}_{\mathrm{avg}}=0.061\pm0.016\) and \(0.99\times\) speedup, confirming that the grouped estimator reduces to the full objective up to path-by-path overhead. Part of the observed advantage over direct leaf-path sampling is implementation-specific: direct leaf paths are sampled with replacement, while terminal groups are sampled without replacement. On a \(16\)-leaf tree, replacement creates duplicate sampled leaves more often, whereas terminal sampling avoids this redundancy.

\begin{figure}[ht]
    \centering
    \begin{minipage}{0.49\textwidth}
        \centering
        \includegraphics[width=\linewidth]{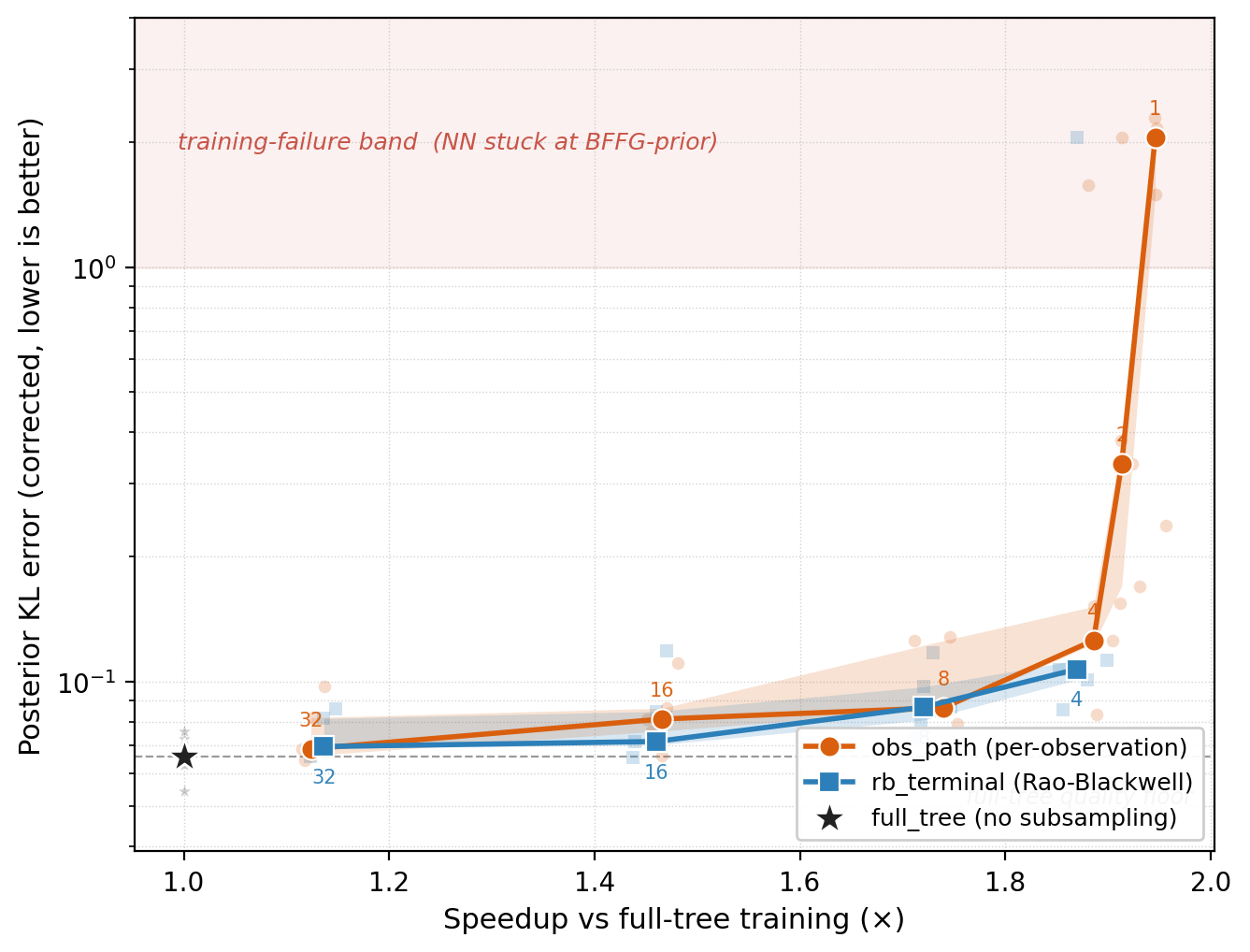}
        \textbf{(a) Benchmark 1C}
    \end{minipage}
    \hfill
    \begin{minipage}{0.49\textwidth}
        \centering
        \includegraphics[width=\linewidth]{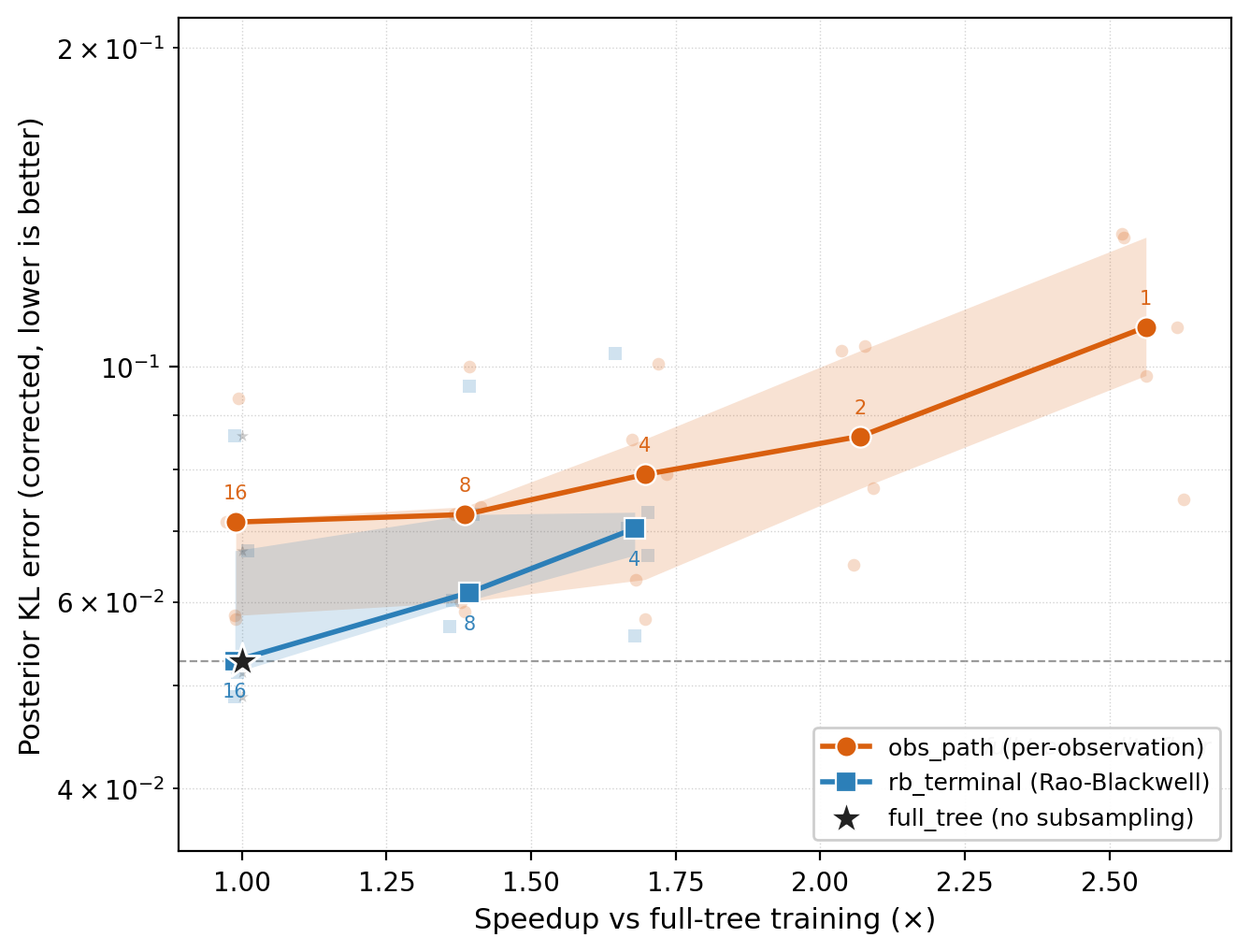}
        \textbf{(b) Benchmark 1D}
    \end{minipage}
    \caption{Accuracy--speed trade-offs for the path-subsampling benchmarks. The horizontal axis is speedup relative to full-tree training and the vertical axis is average marginal KL to the exact smoothing posterior. Benchmark~1C shows a sharper accuracy--speed trade-off, while Benchmark~1D remains stable under aggressive subsampling and gives larger speedups because each skipped OU path avoids SDE simulation.}
    \label{fig:exp1_subsampling_scaling}
\end{figure}

\subsection{Benchmark 2A: Folded Nonlinear Gaussian Tree}
\label{app:exp2a_details}

Benchmark~2A uses a small nonlinear discrete tree designed to create a four-mode root posterior while keeping the reference calculation lightweight. The inferred stochastic root $\rho$ is attached to the fixed super-root $\0$ and has prior \(X_\rho\sim\mathcal{N}(0,\tau_0^2I_2)\) with \(\tau_0=1.5\). In the implementation, this prior is represented by a Gaussian edge from the fixed dummy parent into $\rho$. The inferred root has four terminal latent children, each with transition
\begin{equation}
    X_v=f(X_\rho)+\epsilon_v,\qquad f(x)=(x_1^2,x_2^2),\qquad \epsilon_v\sim\mathcal{N}(0,q^2I_2),
    \label{eq:exp2a_folded_transition}
\end{equation}
where \(q=0.08\). Each terminal latent vertex has one observed child with \(X_l\mid X_{\pa{l}}\sim\mathcal{N}(X_{\pa{l}},r_{\mathrm{obs}}^2I_2)\), \(r_{\mathrm{obs}}=0.05\), and all observed leaves are fixed at \(x_l=(1,1)\). The augmented implementation tree has \(10\) vertices: one fixed dummy parent $\0$, one inferred stochastic root $\rho$, four terminal latent vertices, and four observed leaves. Figure~\ref{fig:exp2a_tree} shows the topology.

\begin{figure}[ht]
    \centering
    \includegraphics[width=0.55\textwidth]{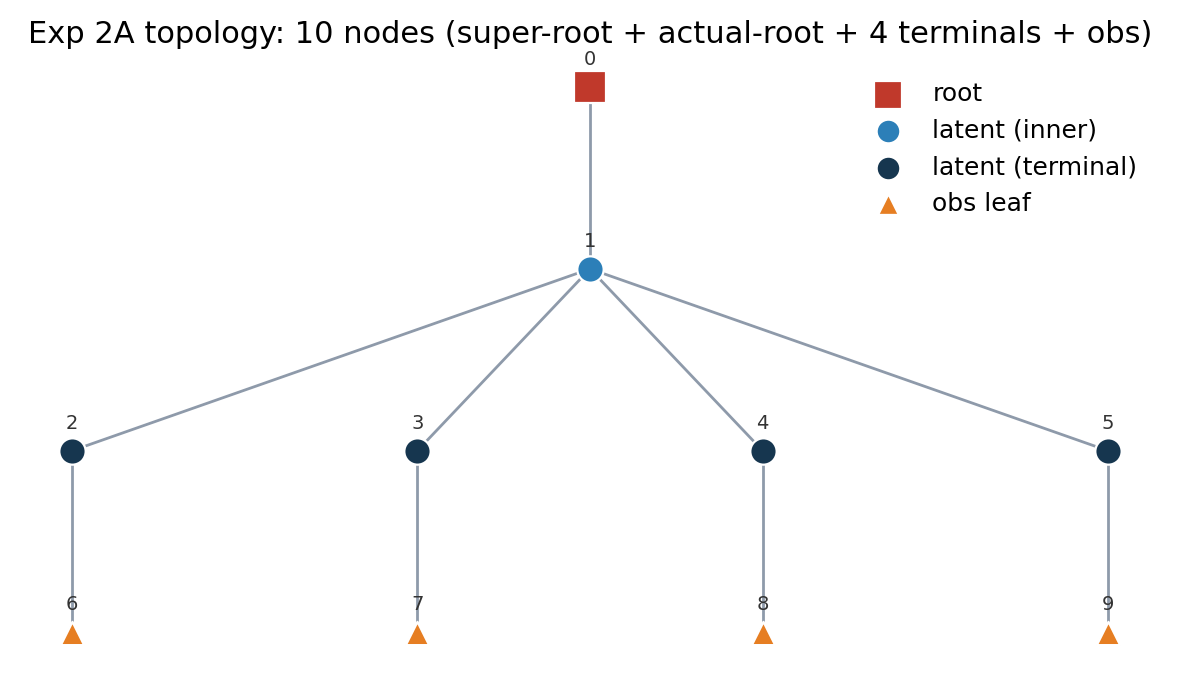}
    \caption{Tree topology for Benchmark~2A. The fixed dummy parent $\0$ implements the Gaussian prior on the inferred stochastic root $\rho$.}
    \label{fig:exp2a_tree}
\end{figure}

The reference posterior for $X_\rho$ is computed by analytically integrating out the terminal latent vertices. Since
\begin{equation}
    x_l\mid X_\rho=x\sim\mathcal{N}(f(x),(q^2+r_{\mathrm{obs}}^2)I_2),
    \label{eq:exp2a_marginal_likelihood}
\end{equation}
we evaluate
\begin{equation}
    p(x\mid x_\L)\propto \mathcal{N}(x;0,\tau_0^2I_2)\prod_{l\in\L}\mathcal{N}(x_l;f(x),(q^2+r_{\mathrm{obs}}^2)I_2)
    \label{eq:exp2a_grid_reference}
\end{equation}
on a \(201\times201\) grid over \([-3,3]^2\), normalize numerically, and use the resulting grid distribution to estimate root-quadrant probabilities and draw reference samples. The four reference quadrant probabilities are \(0.25\) up to numerical precision.

Let \(\mathcal{Q}=\{Q_{++},Q_{+-},Q_{-+},Q_{--}\}\), where each quadrant is defined by the signs of \((x_1,x_2)\). Given root samples \(\{x_\rho^{(i)}\}_{i=1}^n\) from a method, define the empirical quadrant probability
\begin{equation}
    \hat{q}(Q)=\frac{1}{n}\sum_{i=1}^n\mathds{1}\{x_\rho^{(i)}\in Q\},\qquad Q\in\mathcal{Q}.
    \label{eq:exp2a_empirical_quadrant}
\end{equation}
The quadrant-probability error is
\begin{equation}
    E_{\mathrm{quad}}=\sum_{Q\in\mathcal{Q}}|\hat{q}(Q)-p_{\mathrm{grid}}(Q)|,
    \label{eq:exp2a_quad_error}
\end{equation}
which measures whether the method assigns the correct probability mass to the four sign modes. Because the grid reference is symmetric in this benchmark, \(p_{\mathrm{grid}}(Q)=0.25\) for all four quadrants up to numerical precision. The maximum possible value is \(1.5\), attained by a method that puts all mass in one quadrant.

The recovered-mode count is
\begin{equation}
    M_{\mathrm{rec}}=\sum_{Q\in\mathcal{Q}}\mathds{1}\{\hat{q}(Q)\ge 0.01\}.
    \label{eq:exp2a_recovered_modes}
\end{equation}
It is a coarse mode-coverage diagnostic: \(M_{\mathrm{rec}}=4\) means that all four sign modes are represented by the sample set, but it does not by itself imply the correct mode weights. This is why Table~\ref{tab:exp2a_multimodal} reports both recovered modes and \(E_{\mathrm{quad}}\).

For density-level comparison, let \(\hat{p}^{\mathrm{hist}}\) and \(p_{\mathrm{grid}}^{\mathrm{hist}}\) be the method and grid-reference root distributions binned on the same two-dimensional histogram. The reported JS divergence is
\begin{equation}
    \mathrm{JS}(\hat{p}^{\mathrm{hist}},p_{\mathrm{grid}}^{\mathrm{hist}})
    =
    \frac{1}{2}\mathrm{KL}(\hat{p}^{\mathrm{hist}}\Vert m)
    +
    \frac{1}{2}\mathrm{KL}(p_{\mathrm{grid}}^{\mathrm{hist}}\Vert m),
    \qquad
    m=\frac{1}{2}(\hat{p}^{\mathrm{hist}}+p_{\mathrm{grid}}^{\mathrm{hist}}).
    \label{eq:exp2a_js}
\end{equation}
It measures global discrepancy of the root marginal density, not only quadrant mass. Finally, given method samples \(x_i\), grid-reference samples \(y_i\), and random unit directions \(u_r\in\mathbb{S}^1\), the reported sliced Wasserstein distance is
\begin{equation}
    \mathrm{SW}_2
    =
    \left\{
    \frac{1}{R}\sum_{r=1}^{R}
    W_2^2\!\left(
        \{\langle u_r,x_i\rangle\}_{i=1}^n,
        \{\langle u_r,y_i\rangle\}_{i=1}^n
    \right)
    \right\}^{1/2}.
    \label{eq:exp2a_sw}
\end{equation}
It is sensitive to displacement of mass within and across quadrants. Lower values of JS and SW\(_2\) indicate closer agreement with the grid-reference root posterior.

For NBFFG GMM, four of the five independent training seeds recover all four modes. One run recovers three modes, producing the nonzero standard deviation in the recovered-mode count and the larger standard deviation in \(E_{\mathrm{quad}}\), JS, and SW\(_2\). We report this directly because it reflects residual optimization sensitivity of the mixture correction to network initialization.

\begin{table}[t]
\centering
\caption{Benchmark 2A: folded nonlinear Gaussian tree. Values are mean \(\pm\) std. over five runs. The grid reference has four equiprobable root-quadrant modes. Lower is better for quadrant error, JS, and sliced \(W_2\); higher is better for recovered modes.}
\label{tab:exp2a_multimodal}
\scriptsize
\setlength{\tabcolsep}{4pt}
\begin{tabular}{@{}lrrrrr@{}}
\toprule
Method & \(K_{\mathrm{mix}}\) & Modes \(\uparrow\) & Quad. err. \(\downarrow\) & JS \(\downarrow\) & SW\(_2\) \(\downarrow\) \\
\midrule
Prior & -- & \(4.0\pm0.0\) & \(0.013\pm0.004\) & \(0.669\pm0.002\) & \(0.493\pm0.005\) \\
Guide only & -- & \(1.0\pm0.0\) & \(1.500\pm0.000\) & \(0.406\pm0.001\) & \(1.126\pm0.000\) \\
Unguided GMM & \(4\) & \(2.2\pm0.4\) & \(0.903\pm0.209\) & \(0.626\pm0.012\) & \(0.560\pm0.077\) \\
NBFFG Gaussian & \(1\) & \(1.0\pm0.0\) & \(1.500\pm0.000\) & \(0.406\pm0.001\) & \(1.124\pm0.001\) \\
NBFFG GMM & \(4\) & \(\mathbf{3.8\pm0.4}\) & \(\mathbf{0.118\pm0.219}\) & \(\mathbf{0.075\pm0.043}\) & \(\mathbf{0.094\pm0.143}\) \\
\bottomrule
\end{tabular}
\end{table}

The canonical guide uses the random-walk proxy \(\tilde{X}_v=X_\rho+\tilde{\epsilon}_v\), \(\tilde{\epsilon}_v\sim\mathcal{N}(0,q^2I_2)\), while keeping the true observation model. We compare the prior, the uncorrected guide, an unguided Gaussian mixture, NBFFG with one Gaussian component, and NBFFG with four mixture components. All trained models use a three-layer correction network with embedding dimension \(8\), \(32\) particles per iteration, \(5{,}000\) training iterations, peak learning rate \(10^{-3}\), \(200\) warmup steps, cosine decay to \(30\%\) of the peak learning rate, and gradient clipping at \(5.0\). Evaluation uses \(16\) batches of \(128\) particles for NELBO estimation and \(8{,}192\) root samples for posterior diagnostics.

\subsection{Benchmark 2B: Double-Well Diffusion Tree}
\label{app:exp2b_details}

Benchmark~2B uses a one-dimensional double-well diffusion on each hidden edge,
\begin{equation}
    dZ_v(t)=-4\alpha Z_v(t)(Z_v(t)^2-1)dt+\sigma dW_t,\qquad Z_v(0)=X_{\pa{v}},\qquad X_v=Z_v(T_v).
    \label{eq:exp2b_double_well_sde}
\end{equation}
The potential wells are centered at \(-1\) and \(1\). Each observed leaf is attached to a terminal latent vertex with \(X_l\mid X_{\pa{l}}\sim\mathcal{N}(X_{\pa{l}},r_{\mathrm{obs}}^2)\). The root is fixed at \(X_\0=0\). The latent tree has two root children and four terminal latent vertices; after attaching one observed leaf to each terminal latent vertex, the augmented tree has \(11\) vertices. The two root-child edges have length \(T_{\mathrm{root}}=4.0\), and the terminal latent edges have length \(T_{\mathrm{leaf}}=1.0\).

The numerical parameters are
\begin{equation}
    \alpha=3.0,\qquad \sigma=0.5,\qquad r_{\mathrm{obs}}=0.1.
    \label{eq:exp2b_parameters}
\end{equation}
All continuous methods use the same Euler--Maruyama discretization with \(100\) solver steps per edge. The canonical proxy is Brownian motion with the true diffusion coefficient, \(d\tilde{Z}_v(t)=\sigma d\tilde{W}_t\), together with the true Gaussian observation model. This proxy is intentionally misspecified in drift while remaining path-measure compatible with the true diffusion.

We use two deterministic observation regimes. In the early-commitment regime, the four observed leaves are fixed at
\begin{equation}
    x_\L=(-1,-1,1,1),
    \label{eq:exp2b_early_commit_obs}
\end{equation}
so the two main branches should commit early to opposite wells. In the bimodal regime,
\begin{equation}
    x_\L=(-1,-1,-1,1),
    \label{eq:exp2b_bimodal_obs}
\end{equation}
which leaves competing explanations for the major branch assignments and produces a substantially harder path-space inference problem.

NBFFG uses a five-layer residual control network with hidden dimension \(32\), structural embedding dimension \(8\), \(32\) training particles per iteration, \(10{,}000\) iterations, peak learning rate \(10^{-3}\), \(500\) warmup steps, cosine decay to \(50\%\) of the peak learning rate, and gradient clipping at \(5.0\). Guided SMC uses \(1024\) particles and \(16\) linear temperature steps. Guided pCN uses one chain with \(5000\) steps, \(1000\) burn-in steps, thinning by \(100\), and proposal scale \(\beta=0.96\). All reported values are computed over five independent training seeds.

For an importance-weighted method with normalized weights \(\bar{w}_1,\ldots,\bar{w}_N\), we report
\begin{equation}
    \mathrm{ESS}=\frac{1}{\sum_{i=1}^{N}\bar{w}_i^2},\qquad \mathrm{ESS/N}=\frac{1}{N\sum_{i=1}^{N}\bar{w}_i^2}.
    \label{eq:exp2b_ess}
\end{equation}
ESS applies to guide only, NBFFG, and guided SMC, but not to guided pCN, for which we instead report the Metropolis acceptance rate (acc.). Since the sample budgets differ across methods, ESS/N is the primary proposal-efficiency diagnostic. Table~\ref{tab:exp2b_double_well} shows that NBFFG improves ESS/N in both regimes, while the bimodal case remains severely degenerate for all importance-based methods and nearly non-mixing for guided pCN.

\begin{table}[htbp]
\centering
\caption{Benchmark 2B: double-well diffusion tree. Values are mean \(\pm\) std. over five training seeds. ESS is reported for importance-weighted proposals and guided SMC; guided pCN is reported by acceptance rate instead.}
\label{tab:exp2b_double_well}
\scriptsize
\setlength{\tabcolsep}{5pt}
\begin{tabular}{@{}llrrrr@{}}
\toprule
Regime & Method & ESS \(\uparrow\) & ESS / particle \(\uparrow\) & pCN acc. \(\uparrow\) & Inference time (s) \\
\midrule
Early commit & Guide only & \(279.6\pm8.9\) & \(0.273\pm0.009\) & -- & \(8.50\pm0.09\) \\
Early commit & Guided SMC & \(277.9\pm16.4\) & \(0.271\pm0.016\) & -- & \(1.39\pm0.01\) \\
Early commit & Guided pCN & -- & -- & \(0.386\pm0.006\) & \(3.60\pm0.03\) \\
Early commit & NBFFG & \(\mathbf{541.6\pm105.5}\) & \(\mathbf{0.529\pm0.103}\) & -- & \(14.72\pm0.01\) \\
\midrule
Bimodal & Guide only & \(3.0\pm3.3\) & \(0.0029\pm0.0032\) & -- & \(4.73\pm0.02\) \\
Bimodal & Guided SMC & \(2.0\pm1.1\) & \(0.0020\pm0.0011\) & -- & \(0.94\pm0.01\) \\
Bimodal & Guided pCN & -- & -- & \(0.0015\pm0.0015\) & \(2.89\pm0.03\) \\
Bimodal & NBFFG & \(\mathbf{7.8\pm5.9}\) & \(\mathbf{0.0076\pm0.0058}\) & -- & \(14.07\pm0.06\) \\
\bottomrule
\end{tabular}
\end{table}

\subsection{Experiment 3: Phylogenetic Butterfly Shape Inference}
\label{app:exp3_butterfly_details}

Experiment~3 uses NBFFG as a qualitative real-world application for stochastic phylogenetic shape inference. Classical phylogenetic comparative models often use Brownian or Ornstein--Uhlenbeck dynamics for low-dimensional continuous traits \citep{felsenstein_evolutionarytreesdna_1981,butler_phylogenetic_2004}. Here the state is instead a two-dimensional wing outline, represented by ordered landmarks. We use a Kunita-flow shape diffusion, following recent stochastic shape-process models \citep{sommer_stochasticsshapeskunitaflows_2025,stroustrup_stochastic_2025}, because the induced noise is spatially correlated across nearby landmarks and therefore preserves coherent deformations better than independent Euclidean noise on each coordinate.

Let \(z=(z^1,\ldots,z^n)\), with \(z^i\in\mathbb{R}^2\), denote the landmark representation of a wing outline. Along each hidden phylogenetic edge we use the drift-free shape SDE
\begin{equation}
    dZ_v(t)=K(Z_v(t))dW_v(t),\qquad Z_v(0)=X_{\pa{v}},\qquad X_v=Z_v(T_v),
    \label{eq:exp3_shape_sde}
\end{equation}
where \(K(z)\) is a $2n\times 2n$ block matrix, whose \(2\times2\) landmark block is \([K(z)]_{ij}=\kappa(z^i,z^j)I_2\), with \(\kappa(z^i,z^j)=4\alpha(3+3r+r^2)\exp(-r/\sigma)\) and \(r=\|z^i-z^j\|_2\). We use \(\alpha=0.05\) and \(\sigma=0.2\). Observed extant species are modeled as Gaussian leaf-edge observations, \(X_l\mid X_{\pa{l}}\sim\mathcal{N}(X_{\pa{l}},0.05^2I)\).

The data consist of a Papilio phylogeny with \(17\) observed extant species. After adding a fixed super-root $\0$, the augmented tree has \(34\) vertices and depth \(8\). The fixed super-root state is the Euclidean mean of the observed landmark configurations, and the original phylogenetic root $\rho\in\S$ is treated as the inferred most recent common ancestor. The raw outlines have \(118\) landmarks; we remove one overlapping landmark, keep every second landmark, rescale coordinates by \(0.002\), and obtain \(n=59\) landmarks, hence state dimension \(d=118\). Branch lengths are normalized by the maximum original edge length, and the synthetic super-root edge is assigned the mean normalized edge length.

The proxy has zero drift and uses the same edge lengths as the phylogeny. In the reported run, we use a per-edge proxy diffusion: after a backward proxy pass, the diffusion matrix is evaluated at the Gaussian message mode associated with the child endpoint. This keeps the guide analytically Gaussian while adapting the local diffusion scale to the shape expected at each edge. The neural residual control is then trained on top of this proxy guide.

The variational SDE uses \(50\) Euler--Maruyama steps per edge. The residual control network has \(4\) layers, hidden dimension \(512\), node embedding dimension \(16\), time embedding dimension \(16\), and SiLU activations. We train for \(5000\) iterations using Adam with learning rate \(10^{-3}\), cosine decay to \(10\%\) of the peak learning rate, \(8\) particles per iteration, and gradient clipping at \(5.0\). For visualization, we draw \(200\) posterior endpoint samples and plot posterior means for the internal vertices. This experiment is not used as a quantitative benchmark because no exact posterior is available; it is intended to show that the same NBFFG construction can be applied to high-dimensional nonlinear tree smoothing with real observations.

\section{Proofs}
\label{app:proofs}

\paragraph{Full recursion for Proposition~\ref{prop:gaussian_guiding_messages}.}
The proxy messages in Proposition~\ref{prop:gaussian_guiding_messages} have information forms
\[
    \log\tilde{h}_u(x)
    =
    -\frac{1}{2}x^\top\tilde{H}_u x +\tilde{\eta}_u^\top x + c_u,
    \qquad
    \log\tilde{m}_v(x)
    =
    -\frac{1}{2}x^\top G_vx + \xi_v^\top x + c'_v,
\]
with $\tilde{H}_u=\sum_{v\in\ch{u}}G_v$ and $\tilde{\eta}_u=\sum_{v\in\ch{u}}\xi_v$. The child-message parameters satisfy
\begin{equation}
    (G_v,\xi_v)
    =
    \begin{cases}
        (L_v^\top R_v^{-1}L_v,\; L_v^\top R_v^{-1}(x_v-\beta_v)), & v\in\L,\\
        ((\tilde{H}_v^{-1}+\tilde{\Sigma}_v)^{-1},\;(\tilde{H}_v^{-1}+\tilde{\Sigma}_v)^{-1}\tilde{H}_v^{-1}\tilde{\eta}_v), & v\in\S.
    \end{cases}
    \label{eq:gaussian_message_recursion}
\end{equation}
For a continuous edge, with $\tilde{\Sigma}_{v,t}=\int_t^{T_v}\tilde{\sigma}_v(\tau)\tilde{\sigma}_v(\tau)^\top \di \tau$,
\[
    \log\tilde{h}_{v,t}(z) = -\frac{1}{2}z^\top\tilde{H}_{v,t}z + \tilde{\eta}_{v,t}^\top z + c_{v,t},
    \qquad
    \tilde{H}_{v,t} = (\tilde{H}_v^{-1}+\tilde{\Sigma}_{v,t})^{-1},
    \qquad
    \tilde{\eta}_{v,t} = \tilde{H}_{v,t}\tilde{H}_v^{-1}\tilde{\eta}_v.
\]

\begin{proof}[Proof of Proposition~\ref{prop:gaussian_guiding_messages}]
The endpoint law of the continuous proxy in Equation~\eqref{eq:continuous_proxy} can be written as $\mathcal{N}(\di y;x,\tilde{\Sigma}_{T_v})$, where
\[
    \tilde{\Sigma}_{T_v} := \int_{0}^{T_v} \tilde{\sigma}_v(t)\tilde{\sigma}_v(t)^{\top}\,\di t.
\]
Thus the discrete and continuous proxies have the same endpoint-kernel form. We use a unified notation $\tilde{\Sigma}_v$ to represent the covariance of both endpoint transitions unless specified. The vertex recursion follows by induction from the leaves. For a leaf $l\in\L$, define the message directly as the observation likelihood
\[
    \tilde{m}_l(x) = p_l(x_l\mid x) = \mathcal{N}(x_l;L_lx+\beta_l,R_l),
\]
which, as a function of $x$, has information parameters $(G_l, \xi_l) = (L_l^\top R_l^{-1}L_l,L_l^\top R_l^{-1}(x_l-\beta_l))$. If
\[
    \log \tilde{h}_v(y) = -\frac{1}{2}y^\top \tilde{H}_v y + \tilde{\eta}_v^\top y + c_v
\]
for an internal child $v\in\S$, then $\tilde{h}_v(y)\propto \mathcal{N}(y;\tilde{H}_v^{-1}\tilde{\eta}_v,\tilde{H}_v^{-1})$ and
\[
    \tilde{m}_v(x) = (\tilde{P}_v\tilde{h}_v)(x) \propto \int \mathcal{N}(y;\tilde{H}_v^{-1}\tilde{\eta}_v,\tilde{H}_v^{-1}) \mathcal{N}(\di y;x,\tilde{\Sigma}_v) = \mathcal{N}(\tilde{H}_v^{-1}\tilde{\eta}_v;x,\tilde{H}_v^{-1}+\tilde{\Sigma}_v),
\]
which has information parameters given by $(G_v,\xi_v)$ in Equation~\eqref{eq:gaussian_message_recursion}. Multiplying child messages adds information parameters, giving $\tilde{H}_u=\sum_{v\in\ch{u}}G_v$ and $\tilde{\eta}_u=\sum_{v\in\ch{u}}\xi_v$.

For the within-edge guide, the proxy transition from time $t$ to $T_v$ is $\mathcal{N}(\di y;z,\tilde{\Sigma}_{v,t})$. Therefore
\[
    \tilde{h}_{v,t}(z)
    =
    (\tilde{P}_{v,t\to T_v}\tilde{h}_v)(z)
    \propto
    \int
    \mathcal{N}(y;\tilde{H}_v^{-1}\tilde{\eta}_v,\tilde{H}_v^{-1})
    \mathcal{N}(\di y;z,\tilde{\Sigma}_{v,t})
    =
    \mathcal{N}(\tilde{H}_v^{-1}\tilde{\eta}_v;z,\tilde{H}_v^{-1}+\tilde{\Sigma}_{v,t}).
\]
Its information parameters are $(\tilde{H}_{v,t},\tilde{\eta}_{v,t})$.
\end{proof}

\begin{proof}[Proof of Proposition~\ref{prop:tree_elbo}]
Let
\[
    \P(\di x_\S\mid x_\0)
    =
    \prod_{s\in\S}P_s(x_{\pa{s}},\di x_s)
\]
be the prior law on the unobserved internal vertices. By Equation~\eqref{eq:smoothing_posterior},
\[
    \frac{\di \P^\star}{\di \P}(x_\S)
    =
    \frac{1}{Z(x_\L,x_\0)}
    \prod_{l\in\L}p_l(x_l\mid x_{\pa{l}}).
\]
By the factorization of $\Q^\theta$ and absolute continuity of each edge kernel,
\[
    \frac{\di \Q^\theta}{\di \P}(x_\S)
    =
    \prod_{s\in\S}
    \frac{\di Q_s^\theta(X_{\pa{s}},\cdot)}
         {\di P_s(X_{\pa{s}},\cdot)}(X_s).
\]
Therefore
\[  
    \begin{aligned}
        \kld(\Q^\theta\Vert\P^\star)
        &= \mathbb{E}_{\Q^\theta} \left[\log\left(\frac{\di \Q^{\theta}}{\di \P^{\star}}(X_\S)\right)\right] \\
        &= \mathbb{E}_{\Q^\theta} \left[\log \left(\frac{\di \Q^{\theta}}{\di \P}(X_\S)\cdot\left(\frac{\di \P^{\star}}{\di \P}\right)^{-1}(X_\S)\right)\right] \\
        &= \mathbb{E}_{\Q^\theta} \left[\sum_{s\in\S} \log \frac{\di Q_s^\theta(X_{\pa{s}},\cdot)}{\di P_s(X_{\pa{s}},\cdot)}(X_s) - \sum_{l\in\L} \log p_l(x_l\mid X_{\pa{l}}) \right] + \log Z(x_\L,x_\0).
    \end{aligned}
\]
The bracketed expectation is $\J(\theta)$, proving the claim.
\end{proof}

%%%%%%%%%%%%%%%%%%%%%%%%%%%%%%%%%%%%%%%%%%%%%%%%%%%%%%%%%%%%

% \newpage
% \input{checklist.tex}

\end{document}